\documentclass[twoside,11pt]{article}
\usepackage{bookNotation}
\usepackage{latexsym}
\usepackage{times}
\usepackage{amsmath,math,bbm}
\usepackage[pdftex]{graphicx}
\usepackage[pdftex]{hyperref}
\DeclareGraphicsExtensions{.jpg,.pdf,.eps, .png}
\usepackage{makeidx}
\usepackage{subfigure}
\usepackage{algorithm}
\usepackage{algorithmic}

\def\CTD{\mbox{CTD}}

\title{3D Shape Registration Using Spectral Graph Embedding and Probabilistic Matching\thanks{In \textit{Image Processing and Analysing With Graphs: Theory and Practice}, CRC Press, chapter 15, pp.441-474, 2012}}
\author{Avinash Sharma, Radu Horaud and Diana Mateus\\
Inria Grenoble Rh\^one-Alpes\\
655 avenue de l'Europe \\
38330 Montbonnot Saint-Martin, France\\
}
\date{}
\begin{document}
\maketitle
%\thispagestyle{empty} \vskip 10pt \hrule width \textwidth height
%1.5pt \vskip 10pt

%%%%%%%%%%%%%%%%%%%%%%%%%%%%%%%%%%%%%%%%%%%%%%
%        Chapter files and sections          %
%%%%%%%%%%%%%%%%%%%%%%%%%%%%%%%%%%%%%%%%%%%%%%

%\input{Author1surname_macros}          % Macros used in chapter (if any)
\section*{Abstract}
%\begin{abstract}
In this book chapter we address the problem of 3D shape registration and we propose a novel technique based on spectral graph theory and probabilistic matching. Recent advancement 
in shape acquisition technology has led to the capture of large amounts of 3D data. 
Existing real-time multi-camera 3D acquisition methods provide a frame-wise reliable visual-hull
or mesh representations for real 3D animation sequences 
The task of 3D shape analysis involves tracking, recognition, registration, etc. 
Analyzing 3D data in a single framework is still a challenging task considering
the large variability of the data gathered with different acquisition devices.
3D shape registration is one such challenging shape analysis task.
The main contribution of this chapter is to extend the spectral graph matching methods
to very large graphs by combining spectral graph matching with Laplacian embedding. 
Since the embedded representation of a graph is obtained by dimensionality reduction
we claim that the existing spectral-based methods are not easily applicable. 
We discuss solutions for the exact and
inexact graph isomorphism problems and recall the main spectral properties of the combinatorial graph Laplacian;
We provide a novel analysis of the commute-time embedding that allows us to interpret the latter in
terms of the PCA of a graph, and to select the appropriate dimension of the associated embedded 
metric space;
We derive a unit hyper-sphere normalization for the commute-time embedding that allows us to register
two shapes with different samplings;
We propose a novel method to find the eigenvalue-eigenvector ordering and the eigenvector sign using the
eigensignature (histogram) which is invariant to the isometric shape deformations and fits well in 
the spectral graph matching framework, and
we present a probabilistic shape matching formulation using an expectation maximization
point registration algorithm which alternates between aligning the eigenbases and finding a vertex-to-vertex assignment. 
%\end{abstract}

\section{Introduction} % section title
\label{Sharma_sec:intro}  % \label{} allows reference to this section
%!!! please use unique labels (e.g., include your initials) for all
%your sections, equations, figures, tables, etc.

In this chapter we discuss the problem of 3D shape matching\index{shape matching}. Recent advancement 
in shape acquisition technology has led to the capture of large amounts of 3D data. 
Existing real-time multi-camera 3D acquisition methods provide a frame-wise reliable visual-hull
or mesh representations for real 3D animation sequences 
~\cite{Franco-Boyer-PAMI-2008,Starck-Hilton-CGA-2007,Slabaugh-Malzbender-WVG-2001,Seitz-Brian-CVPR-2006, Vlasic-Baran-SIGGRAPH-2008,ZBH11}. The task of 3D shape analysis involves tracking, recognition, registration, etc. 
Analyzing 3D data in a single framework is still a challenging task considering
the large variability of the data gathered with different acquisition devices.
3D shape registration is one such challenging shape analysis task.
The major difficulties in shape registration arise due to: 1) variation in the shape acquisition 
techniques, 2) local deformations in non-rigid shapes, 3) large acquisition discrepancies 
(\textit{e.g.}, holes, topology change, surface acquisition noise), 4) local scale change. 

Most of the previous attempts of shape matching can be broadly categorized as \textit{extrinsic} 
or \textit{intrinsic} approaches depending on how they analyze the properties of the underlying manifold.
Extrinsic approaches mainly focus on finding a global or local rigid transformation between two 3D shapes. 

There is large set of approaches based on variations of the iterative closest point (ICP)
algorithm~\cite{ChenGerard92,BeslMcKay92,RusinkiewiczLevoy-3DIM-2001} that falls in the category of extrinsic approaches.
However, the majority of these approaches compute rigid transformations for shape registration and are 
not directly applicable to non-rigid shapes. 
Intrinsic approaches are a natural choice for finding dense correspondences between
articulated shapes, as they embed the shape in some canonical domain which preserves
some important properties of the manifold, \textit{e.g.}, geodesics and
angles. Intrinsic approaches are preferable over extrinsic as they provide a global representation which 
is invariant to non-rigid deformations that are common in the real-world 3D shapes. 

Interestingly, mesh representation also enables the adaptation of well established
graph matching algorithms that use eigenvalues and eigenvectors of graph matrices,
and are theoretically well investigated in the framework of \textit{spectral graph theory} (SGT)
\textit{e.g.}, \cite{Umeyama88,Wilkinson70}. 
Existing methods in SGT are mainly theoretical results applied to small graphs and under the premise  
that eigenvalues can be computed exactly.
However, spectral graph matching\index{spectral graph matching} does not easily generalize to very large graphs due to 
the following reasons: 1) eigenvalues are approximately computed using eigen-solvers, 2) eigenvalue multiplicity and 
hence ordering change are not well studied, 3) exact matching is intractable for very large graphs. 
It is important to note that these methods mainly focus on exact graph matching while majority 
of the real-world graph matching applications involve graphs with different cardinality and for which only a 
subgraph isomorphism can be sought. 
%Existing heuristics do not work well. 

The main contribution of this work is to extend the spectral graph methods
to very large graphs by combining spectral graph matching with \textit{Laplacian embedding}. 
Since the embedded representation of a graph is obtained by dimensionality reduction
we claim that the existing SGT methods (\textit{e.g.}, \cite{Umeyama88}) are not easily applicable. 
The major contributions of this work are the following: 1) we discuss solutions for the exact and
inexact graph isomorphism problems and recall the main spectral properties of the combinatorial graph Laplacian,
2) we provide a novel analysis of the commute-time embedding that allows us to interpret the latter in
terms of the PCA of a graph, and to select the appropriate dimension of the associated embedded 
metric space, 
3) we derive a unit hyper-sphere normalization for the commute-time embedding that allows us to register
two shapes with different samplings, 
4) we propose a novel method to find the eigenvalue-eigenvector ordering and the eigenvector signs using the
eigensignatures (histograms) that are invariant to the isometric shape deformations and which fits well in 
the spectral graph matching framework,
5) we present a probabilistic shape matching\index{shape matching} formulation using an \textit{expectation maximization} (EM) framework for implementing a
point registration algorithm which alternates between aligning the eigenbases and finding a vertex-to-vertex assignment. 

%Rewrite
The existing graph matching methods that use intrinsic representations are: 
~\cite{BronsteinBronstein2006,WangWang2007,JainZhang2007,ZengZeng2008,mateus:cvpr2008,
RuggeriPatane2009,Lipman2009,DubrovinaKimmel2010}. 
There is another class of methods that allows to combine intrinsic (geodesics) and extrinsic (appearance)
features and which were previously successfully applied for matching features in pairs of images
~\cite{ScottHiggins91,ShapiroBrady92,LuoHancock2001,WangHankcock2006PR,QiuHancock2007PR,LeordeanuHebert2005,
DuchenneBach2009,TorresaniKolmogorov2008,ZassShahua2008, MacielCosteira2003}.
Some recent approaches apply hierarchical matching to find dense correspondences
\cite{HuangAdams2008,ZengWang2010,SahilliogluYemez2010}.
However, many of these graph matching algorithms suffer from the problem of either computational 
intractability or a lack of proper metric as the 
Euclidean metric is not directly applicable while computing distances on non-rigid shapes.
A recent benchmarking of shape matching methods was performed in \cite{BronsteinBronstein2010a}. 
Recently, a few methods proposed a diffusion framework for the task of shape registration
\cite{Ovsjanikov-Merigot-SGP-2010,Sharma-Horaud-NORDIA-2010,Sharma-Horadu-CVPR-2011}. 

In this chapter we present an intrinsic approach for unsupervised 3D shape registration 
first proposed in~\cite{mateus:cvpr2008,KnossowSMH09}. In the first step, dimensionality reduction is performed 
using the graph Laplacian which allows us to embed a 3D shape in an isometric subspace 
invariant to non-rigid deformations. This leads to an embedded point cloud representation where
each vertex of the underlying graph is mapped to a point in a $K$-dimensional metric space.
Thus, the problem of non-rigid 3D shape registration is transformed into a $K$-dimensional
point registration task. However, before point registration, the two eigen spaces 
need to be correctly aligned. This alignment is critical for the spectral matching methods 
because the two eigen spaces are defined up to the signs and the ordering of the eigenvectors of 
their Laplacian matrices. This is achieved by a novel matching method that uses histograms of 
eigenvectors as eigensignatures.
In the final step, a point registration method based on a variant of 
the expectation-maximization (EM) algorithm \cite{HFYDZ11} is applied in order to register two sets of 
points associated with the Laplacian embeddings of the two shapes. The proposed algorithm
alternates between the estimation of an orthogonal transformation matrix associated with 
the alignment of the two eigen spaces and the computation of probabilistic vertex-to-vertex assignment.  
Figure~\ref{Sharma_fig:overview} presents the overview of the proposed method. 
According to the results summarized in \cite{BronsteinBronstein2010a}, this method 
is one among the best performing unsupervised shape matching algorithms.
% Explicitly list the contribution
%
%%%%%%%%%%%%%%%%%%%%%%%%%%%%%%%%%%%%%%%%%%%Figure 1%%%%%%%%%%%%%%%%%%%%%%%%%%%%%%%%%%%%%%%%%%%
\begin{figure}[t]
\begin{center}
\includegraphics[width=\linewidth]{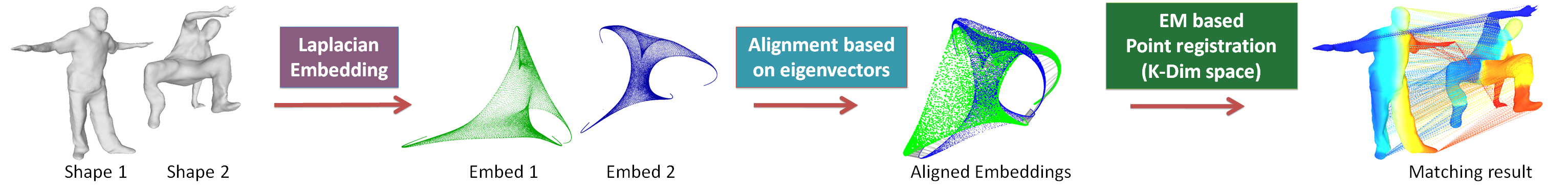} 
\end{center}
\caption{Overview of the proposed method. First, a Laplacian embedding is obtained for each shape. 
Next, these embeddings are aligned to handle the issue of sign flip and ordering change using the histogram 
matching. Finally, an Expectation-Maximization based point registration is performed to obtain 
dense probabilistic matching between two shapes.}
\label{Sharma_fig:overview}
\end{figure}
%%%%%%%%%%%%%%%%%%%%%%%%%%%%%%%%%%%%%%%%%%%%%%%%%%%%%%%%%%%%%%%%%%%%%%%%%%%%%%%%%%%%%%%%%%%%%%%%%

% Section wise roadmap of the chapter
\paragraph{Chapter Overview:}
Graph matrices are introduced in section~\ref{Sharma_section:GraphMatrices}.
The problem of exact graph isomorphism and existing solutions are discussed in 
section~\ref{Sharma_section:SpectralGraphIsomorphism}. 
Section~\ref{Sharma_section:GraphMatchingwithDimReduction} deals with dimensionality 
reduction using the graph Laplacian in order to obtain embedded representations for 3D shapes. 
In the same section we discuss the PCA of graph embeddings and propose a unit hyper-sphere
normalization for these embeddings along with a method to choose the embedding dimension. 
Section~\ref{Sharma_section:AlignmentAndRegistration} introduces the formulation of 
maximum subgraph isomorphism before presenting a two-step method for 3D shape registration. 
In the first step Laplacian embeddings are aligned using histogram matching\index{histogram matching} while in the second 
step we briefly discuss an EM point registration method to obtain probabilistic shape registration.
Finally we present shape matching results in section~\ref{Sharma_section:Results} and 
conclude with a brief discussion in section~\ref{Sharma_section:discussion}.
         % Introduction
\section{Graph Matrices}
\label{Sharma_section:GraphMatrices}
%!!! please use unique labels (e.g., include your initials) for all
%your sections, equations, figures, tables, etc.
%
%In this section we will introduce Graph Matrices
%

A shape can be treated as a connected \textit{undirected weighted graph}
$\G=\{\V,\E\}$ where $\V(\G)=\{v_1,\ldots,v_n\}$ is the vertex set,
$\E(\G)=\{e_{ij}\}$ is the edge set. 
Let $\adjmat$ be the weighted adjacency matrix\index{adjacency matrix}\indexsub{weighted}{adjacency matrix} of this graph. Each
$(i,j)^{\text{th}}$ entry of $\adjmat$ matrix stores weight $w_{ij}$ whenever there is an edge $e_{ij}\in \E(\G)$ between graph 
vertices $v_i$ and $v_j$ and $0$ otherwise with all the diagonal elements set to $0$ . 
We use the following notations:
The degree $d_i$ of a graph vertex $d_i=\sum_{i\sim j}
w_{ij}$ ($i\sim j$ denotes the set of vertices
$v_j$ which are adjacent to $v_i$),
the \textit{degree
  matrix} $\degmat=\diag[d_1 \hdots d_i \hdots d_n]$, the $n\times 1$
vector $\mathbbm{1}=(1\hdots 1)\tp$ (the constant vector), the $n\times 1$ \textit{degree
  vector} $\vec{d}=\degmat\mathbbm{1}$, and the \textit{graph volume}
$\vol(\G)=\sum_i d_i$.

In spectral graph theory, it is
common~\cite{belkin2003laplacian,Luxburg2007} to use the following expression for the edge weights:
\begin{equation}
w_{ij} = e^{-\frac{\text{dist}^2(v_i,v_j)}{\sigma^2}},
\end{equation}
where $\text{dist}(v_i,v_j)$ denotes any distance metric between two vertices and $\sigma$
is a free parameter.
In the case of a \textit{fully connected graph}, matrix $\adjmat$ is
  also referred to as the \textit{similarity matrix}.
The \textit{normalized weighted adjacency matrix} writes
%\begin{equation}
$\tilde{\adjmat} =\degmat^{-1/2}  \adjmat  \degmat^{-1/2}$.
%\end{equation}
The  \textit{transition} matrix of the non-symmetric reversible Markov chain associated
with the graph is
%\begin{equation}
%\label{Sharma_eq:transition-matrix}
$\tilde{\adjmat}_R = \degmat^{-1}  \adjmat = \degmat^{-1/2}  \tilde{\adjmat}  \degmat^{1/2}$.
%\end{equation}

%Each $(i,j)$th entry of the matrix $\tilde{\adjmat}_R$ represents the probability of
%transition in one time step from vertex $v_i$ to vertex $v_j$ and it is
%proportional to the \textit{normalized} edge weight. Hence, a weighted
%undirected
%graph can be viewed as a non-symmetric reversible Markov process. The
%process is symmetric if and only if the graph is regular\footnote[1]{A regular graph is a graph where each vertex has the same number of neighbors; i.e. every vertex has the same degree.}.
%\end{itemize}
%
\subsection{Variants of the Graph Laplacian Matrix}
\label{Sharma_subsection:Laplacians}
We can now build the concept of the \textit{graph Laplacian operator}\index{graph Laplacian}. We
consider the following variants of the Laplacian matrix
\cite{Chung97,Luxburg2007,grady2010discrete}:
\begin{itemize}
\item The \textit{unnormalized Laplacian}\indexsub{unnormalized}{graph Laplacian} which is also referred to as
  the \textit{combinatorial Laplacian}\indexsub{combinatorial}{graph Laplacian} $\lapmat$,
\item the \textit{normalized Laplacian}\indexsub{normalized}{graph Laplacian}  $\normlapmat$, and
\item the \textit{random-walk Laplacian}\indexsub{random-walk}{graph Laplacian} $\normlapmat_R$ also referred to as the
\textit{discrete Laplace operator}.
\end{itemize}
In more detail we have:
\begin{eqnarray}
\label{Sharma_eq:L-combinatorial}
\lapmat &= &  \degmat-\adjmat\\
\label{Sharma_eq:L-normalized}
\normlapmat &= &
 \degmat^{-1/2}  \lapmat  \degmat^{-1/2} = 
\mat{I}-\tilde{\adjmat} \\
\label{Sharma_eq:L-random}
\normlapmat_R &=&
 \degmat^{-1}  \lapmat =  
\mat{I}-\tilde{\adjmat}_R
\end{eqnarray}
with the following relations between these matrices:
\begin{eqnarray}
\label{Sharma_eq:comb-from-rest}
\lapmat &=& \degmat^{1/2}  \normlapmat \degmat^{1/2} = \degmat \normlapmat_R\\
\label{Sharma_eq:normal-from-rest}
\normlapmat &=& \degmat^{-1/2}  \lapmat \degmat^{-1/2} =  \degmat^{1/2}
\normlapmat_R \degmat^{-1/2}\\
\label{Sharma_eq:random-from-rest}
\normlapmat_R &=& \degmat^{-1/2} \normlapmat \degmat^{1/2} =  \degmat\inverse \lapmat.
\end{eqnarray}
%These variants of the Laplacian matrix are related to each other and can be obtained by choosing different weighting schemes (see ~\cite{grady2010discrete} for details). 
%These matrices will be useful to obtain spectral representation of 3D shape in further text. 
        % Section 2
\section{Spectral Graph Isomorphism} % section title
\label{Sharma_section:SpectralGraphIsomorphism}  % \label{} allows reference to this section
%!!! please use unique labels (e.g., include your initials) for all
%your sections, equations, figures, tables, etc.
%
%In this section we will introduce the spectral graph matching methods in detail. The basic formulation of graph matching problem is based on the idea of exact graph isomorphism. 

Let $\G_A$ and  $\G_B$ be two \textit{undirected weighted graphs} with the same
number of nodes, $n$, and let $\adjmat_A$ and $\adjmat_B$ be their adjacency
matrices. They are real-symmetric matrices. In the general case, the
number $r$ of distinct eigenvalues of these matrices is smaller than $n$. The standard spectral methods
only apply to those graphs whose adjacency matrices have $n$
distinct eigenvalues (each eigenvalue has multiplicity one), which
implies 
that the eigenvalues can be ordered.

Graph isomorphism\index{graph isomorphism} \cite{GodsilRoyle2001} can be written as the following minimization problem:
\begin{equation}\label{Sharma_eq:graph-isomorphism}
\mat{P}^{\star} = \arg \min _{\mat{P}} \frobenius {\adjmat_A - \mat{P} \adjmat_B \mat{P}\tp}^2
\end{equation}
where $\mat{P}$ is an $n\times n$ permutation matrix (see
appendix~\ref{Sharma_appendix:permutation}) with $\mat{P}^{\star}$ as the desired vertex-to-vertex permutation matrix and $\frobenius{\bullet}$ is the
Frobenius norm defined by (see appendix~\ref{Sharma_appendix:Frobenius}):
\begin{equation}\label{Sharma_eq:Frobenius}
\| \adjmat \|^2_F = \langle \adjmat,\adjmat \rangle= \sum_{i=1}^{n}\sum_{j=1}^{n} {w}_{ij}^2 = \trace(\adjmat\tp\adjmat) 
\end{equation}
Let:
\begin{eqnarray}
\label{Sharma_eq:A-eigen}
\adjmat_A&=&\mat{U}_A\mat{\Lambda}_A\mat{U}_A\tp\\
\label{Sharma_eq:B-eigen}
\adjmat_B&=&\mat{U}_B\mat{\Lambda}_B\mat{U}_B\tp 
\end{eqnarray}
be the eigen-decompositions
of the two matrices with $n$ eigenvalues $\mat{\Lambda}_A=\diag[\alpha_i]$ and
$\mat{\Lambda}_B=\diag[\beta_i]$ and $n$ orthonormal eigenvectors, the
column vectors of $\mat{U}_A$ and $\mat{U}_B$.

\subsection{An Exact Spectral Solution}
\label{Sharma_subsection:exact-solution}

If there exists a vertex-to-vertex correspondence that makes (\ref{Sharma_eq:graph-isomorphism}) equal to $0$, we have:
\begin{equation}\label{Sharma_eq:exact-iso}
\adjmat_A = \mat{P}^{\star}  \adjmat_B{ \mat{P}^{\star} }\tp.
\end{equation}

This implies that the adjacency matrices of the two graphs should have
the same eigenvalues. Moreover, if the eigenvalues are non null and,
the matrices $\mat{U}_A$ and $\mat{U}_B$ have full rank and are
uniquely defined by their $n$ orthonormal column vectors (which are
the eigenvectors of $\adjmat_A$ and $\adjmat_B$), then 
$\alpha_i=\beta_i, \forall i,\;1\leq i \leq n$ and
$\mat{\Lambda}_A=\mat{\Lambda}_B$. From (\ref{Sharma_eq:exact-iso})
and using the eigen-decompositions of the two graph matrices we obtain:
\begin{equation}
\mat{\Lambda}_A = \mat{U}_A\tp \mat{P}^{\star}  \mat{\breve{U}}_B \mat{\Lambda}_B  \mat{\breve{U}}_B\tp{\mat{P}^{\star}} \tp\mat{U}_A = \mat{\Lambda}_B,
\end{equation}
where the matrix $\mat{\breve{U}}_B$ is defined by:
\begin{equation}
\mat{\breve{U}}_B = \mat{U}_B \mat{S}.
\end{equation}
Matrix $\mat{S}=\diag[s_i]$, with $s_i=\pm 1$, is referred to
as a sign matrix with the property $\mat{S}^2=\mat{I}$. Post
multiplication of $\mat{U}_B$ with a sign matrix
takes into account the fact that the eigenvectors (the column vectors
of $\mat{U}_B$) are only defined up to a sign.
Finally we obtain the following permutation matrix:
\begin{equation}\label{Sharma_eq:exact-perm}
\mat{P}^{\star}  = \mat{U}_B \mat{S} \mat{U}_A\tp.
\end{equation}
Therefore, one may notice that there are as many solutions as the cardinality
of the set of matrices $\mat{S}_n$, i.e., $|\mat{S}_n|=2^n$, and that \textit{not all of these solutions
correspond to a permutation matrix}.
This means that there exist some matrices $\mat{S}^{\star}$ that exactly make
$\mat{P}^{\star}$ a permutation matrix. Hence, all those permutation
matrices that satisfy (\ref{Sharma_eq:exact-perm}) are solutions of the exact
graph isomorphism problem.
Notice that once the permutation has been estimated, one can write that the rows of $\mat{U}_B$ can be
aligned with the rows of $\mat{U}_A$:
\begin{equation}
 \label{Sharma_eq:exact-alignment}
 \mat{U}_A = \mat{P}^{\star}   \mat{U}_B \mat{S}^{\star}.
 \end{equation}
The rows of $\mat{U}_A$ and of $\mat{U}_B$ can be interpreted as 
isometric embeddings of the two graph vertices: A vertex $v_i$ of
$\G_A$ has as
coordinates the $i^\text{th}$ row of $\mat{U}_A$. This means that the
spectral graph
isomorphism\index{spectral graph isomorphism} problem becomes a point registration\index{point registration} problem, where graph
vertices are represented by points in $\mathbbm{R}^n$.
To conclude, the exact graph isomorphism problem has a spectral
solution based on: (i)~the eigen-decomposition of the two graph matrices, (ii)~the
ordering of their eigenvalues, and (iii)~the choice of a sign for
each eigenvector.

\subsection{The Hoffman-Wielandt Theorem}
\label{Sharma_subsection:Hoffman}

The Hoffman-Wielandt theorem\index{Hoffman-Wielandt theorem} \cite{HoffmanWielandt53,Wilkinson65} is
the fundamental building block of spectral graph isomorphism\index{spectral graph isomorphism}. The
theorem holds for normal matrices; Here, we restrict the analysis to real symmetric
matrices, although the generalization to Hermitian matrices is
straightforward:
\begin{theorem}
\label{Sharma_theorem:hoffman-wielandt}
(Hoffman and Wielandt)
If $\adjmat_A$ and $\adjmat_B$ are real-symmetric matrices, and
if $\alpha_i$ and $\beta_i$
are their eigenvalues  arranged in increasing order, $\alpha_1 \leq
\hdots \leq \alpha_i \leq \hdots \leq \alpha_n$ and $\beta_1 \leq \hdots \leq \beta_i \leq \hdots \leq \beta_n$,
%that $\sum_{i=1}^{n}|\alpha_i-\beta_i|^2$ is a minimum for all
 %possible orderings, 
then
\begin{equation}
\label{Sharma_eq:Wielandt-Hoffman}
\sum_{i=1}^{n}(\alpha_i-\beta_i)^2 \leq \frobenius{\adjmat_A - \adjmat_B}^2.
\end{equation}
\end{theorem}
\begin{proof}
The proof is derived from \cite{Wilkinson70,HornJohnson94}. Consider
the eigen-decompositions of matrices $\adjmat_A$ and $\adjmat_B$, (\ref{Sharma_eq:A-eigen}), (\ref{Sharma_eq:B-eigen}).
Notice that for the time being we are free to prescribe the ordering of the eigenvalues
$\alpha_i$ and $\beta_i$ and hence the ordering of the column vectors
of matrices $\mat{U}_A$ and $\mat{U}_B$. By combining (\ref{Sharma_eq:A-eigen}) and (\ref{Sharma_eq:B-eigen}) we write:
\begin{equation}
\mat{U}_A\mat{\mat{\Lambda}}_A\mat{U}_A\tp - \mat{U}_B\mat{\mat{\Lambda}}_B\mat{U}_B\tp = \adjmat_A - \adjmat_B
\end{equation}
or, equivalently:
\begin{equation}
\mat{\mat{\Lambda}}_A\mat{U}_A\tp\mat{U}_B - \mat{U}_A\tp\mat{U}_B\mat{\mat{\Lambda}}_B = \mat{U}_A\tp (\adjmat_A - \adjmat_B)\mat{U}_B.
\end{equation}
By the unitary-invariance of the Frobenius norm (see appendix~\ref{Sharma_appendix:Frobenius} ) and with the notation $\mat{Z} = \mat{U}_A\tp\mat{U}_B$ we obtain:
\begin{equation}
\frobenius{\mat{\mat{\Lambda}}_A \mat{Z} - \mat{Z} \mat{\mat{\Lambda}}_B}^2 = \frobenius{\adjmat_A - \adjmat_B}^2,
\end{equation}
which is equivalent to:
\begin{equation}
\label{Sharma_eq:frobenius-difference}
 \sum_{i=1}^{n}\sum_{j=1}^{n} ( \alpha_i-\beta_j)^2 z_{ij}^2= \frobenius{\adjmat_A - \adjmat_B}^2.
\end{equation}
The coefficients $x_{ij} = z_{ij}^2$ can be viewed as the entries of
a doubly-stochastic matrix $\mat{X}$: $x_{ij} \geq 0,
\sum_{i=1}^{n}x_{ij}=1, \sum_{j=1}^{n}x_{ij}=1$. Using these
properties, we obtain:
\begin{eqnarray}
 \sum_{i=1}^{n}\sum_{j=1}^{n} ( \alpha_i-\beta_j)^2 z_{ij}^2 &=&
 \sum_{i=1}^{n}  \alpha_i^2 + \sum_{j=1}^{n}\beta_j^2 - 2
 \sum_{i=1}^{n}\sum_{j=1}^{n} z_{ij}^2 \alpha_i\beta_j \nonumber \\ 
\label{Sharma_eq:mydevelopment}
&\geq&
 \sum_{i=1}^{n}  \alpha_i^2 + \sum_{j=1}^{n}\beta_j^2 - 2 \max_{Z}\left\{
 \sum_{i=1}^{n}\sum_{j=1}^{n} z_{ij}^2 \alpha_i\beta_j\right\}.
\end{eqnarray}

Hence, the minimization of (\ref{Sharma_eq:frobenius-difference}) is
equivalent to the maximization of the last term in
(\ref{Sharma_eq:mydevelopment}). We can modify our maximization problem to
admit all the doubly-stochastic matrices. In this way we seek an
extremum over a convex compact set. The maximum over this compact set
is larger than or equal to our maximum:
\begin{equation}
\max_{Z\in\mathcal{O}_n}\left\{
 \sum_{i=1}^{n}\sum_{j=1}^{n} z_{ij}^2 \alpha_i\beta_j\right\} \leq 
\max_{X\in\mathcal{D}_n}\left\{
 \sum_{i=1}^{n}\sum_{j=1}^{n} x_{ij} \alpha_i\beta_j\right\}
\end{equation}
where $\mathcal{O}_n$ is the set of orthogonal matrices and
$\mathcal{D}_n$ is the set of  doubly stochastic matrices (see appendix~\ref{Sharma_appendix:permutation}).
Let $c_{ij}=\alpha_i\beta_j$ and hence one can write that the right
term in the equation above as the dot-product of two matrices:
\begin{equation}
\langle \mat{X},\mat{C} \rangle =  \trace(\mat{X}\mat{C}) =  \sum_{i=1}^{n}\sum_{j=1}^{n} x_{ij}
c_{ij}.
\end{equation}
\begin{figure}[h!t]
\centering
\includegraphics[width=0.60\columnwidth]{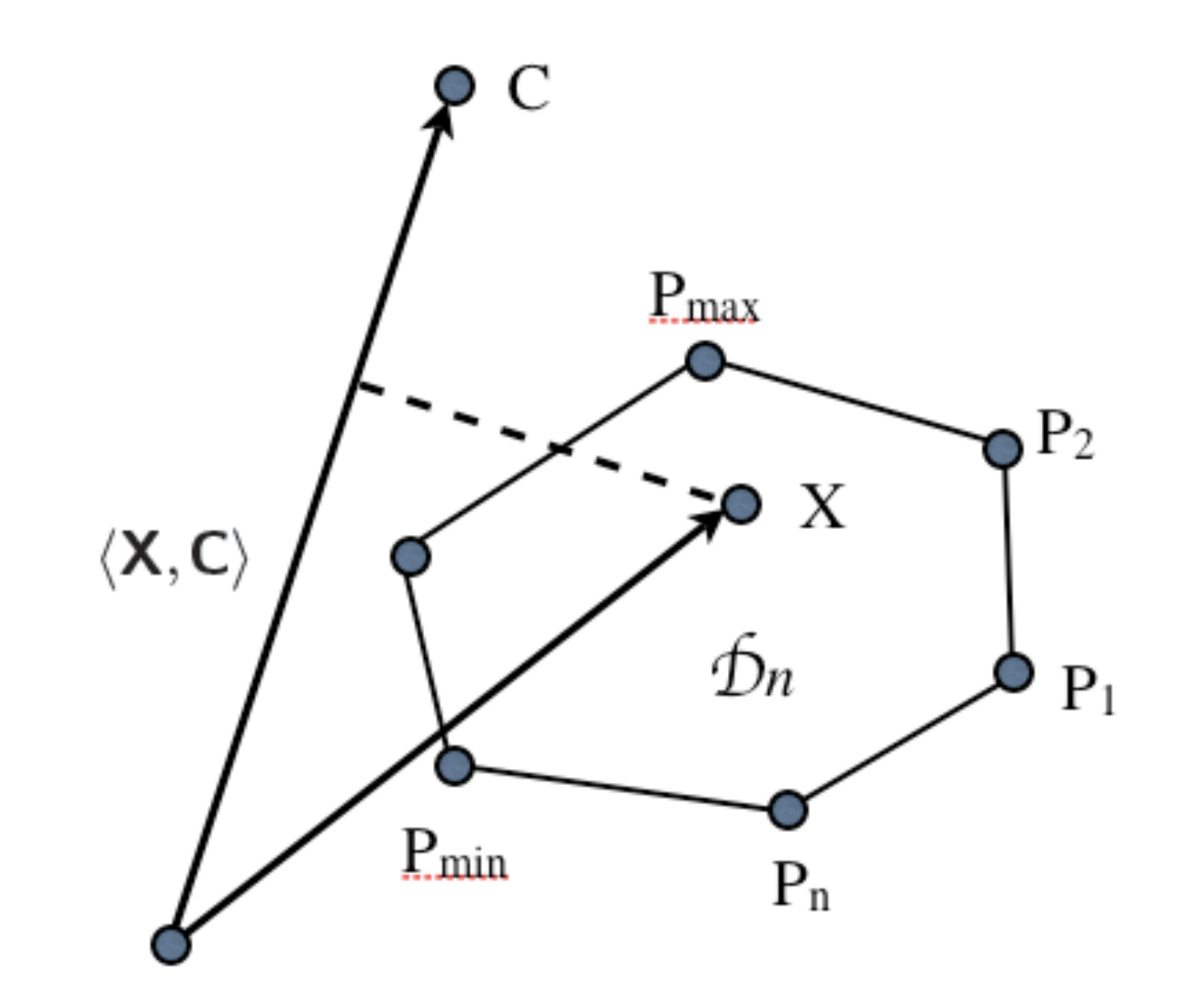}
\caption{This figure illustrates the maximization of the dot-product
$\langle \mat{X},\mat{C} \rangle$. The two matrices can be viewed
as vectors of dimension $n^2$. Matrix $\mat{X}$ belongs to a
compact convex set whose extreme points are the permutation
matrices $\mat{P}_1,\mat{P}_2,\hdots,\mat{P}_n$. Therefore, the
projection of this set (i.e.,
$\mathcal{D}_n$) onto $\mat{C}$ has projected permutation matrices at its extremes, namely $\langle \mat{P}_{\min},\mat{X}\rangle$ and $\langle \mat{P}_{\max},\mat{X}\rangle$ in this example.
}
\label{figure:hoffman}
\end{figure}
Therefore, this expression can be interpreted as the projection of
$\mat{X}$ onto $\mat{C}$, see figure~\ref{figure:hoffman}.
The Birkhoff theorem (appendix~\ref{Sharma_appendix:permutation}) tells us
that the set $\mathcal{D}_n$ of doubly stochastic matrices is a compact convex set.
We obtain that the extrema (minimum and maximum) of the projection of $\mat{X}$ onto $\mat{C}$
occur at the projections of one of the extreme points of this convex set, which correspond to permutation
matrices. Hence, the maximum of $\langle \mat{X},\mat{C} \rangle$ is $\langle \mat{P}_{\max},\mat{X}\rangle$ and we obtain:
\begin{equation}
\max_{X\in\mathcal{D}_n}\left\{
 \sum_{i=1}^{n}\sum_{j=1}^{n} x_{ij} \alpha_i\beta_j\right\} = \sum_{i=1}^{n} \alpha_i\beta_{\pi(i)}.
\end{equation}
By substitution in (\ref{Sharma_eq:mydevelopment}) we obtain:
\begin{equation}
 \sum_{i=1}^{n}\sum_{j=1}^{n} ( \alpha_i-\beta_j)^2 z_{ij}^2 \geq
 \sum_{i=1}^{n} ( \alpha_i-\beta_{\pi(i)})^2. 
\end{equation}
If the eigenvalues are in increasing order then
the permutation that satisfies theorem \ref{Sharma_eq:Wielandt-Hoffman} is the identity matrix,
i.e., $\pi(i) = i$. Indeed, let's assume that for some indices $k$ and
$k+1$ we have: $\pi(k)=k+1$ and $\pi(k+1)=k$. Since $\alpha_k\leq\alpha_{k+1}$ and $\beta_k\leq\beta_{k+1}$, the following inequality
holds:
\begin{equation}
(\alpha_k - \beta_k)^2 + (\alpha_{k+1} - \beta_{k+1})^2 \leq
(\alpha_k - \beta_{k+1})^2 + (\alpha_{k+1} - \beta_{k})^2
\end{equation}
and hence (\ref{Sharma_eq:Wielandt-Hoffman}) holds.  $\quad\blacksquare$
%
%We choose the ordering which minimizes $\sum_{i=1}^{n}|\alpha_i-\beta_i|^2$. 
%It is straightforward to prove that the orderings in
%(\ref{Sharma_eq:order-alpha}) and (\ref{Sharma_eq:order-beta}) give the minimal
%value. This implies the associations $\alpha_i \leftrightarrow
%\beta_i$. 
%
\end{proof}

\begin{corollary}
\label{corollary:first-hoffman-wielandt}
The inequality (\ref{Sharma_eq:Wielandt-Hoffman}) becomes an equality when the eigenvectors of $\adjmat_A$ are aligned with the eigenvectors of $\adjmat_B$ up to a sign ambiguity:
\begin{equation}
\label{Sharma_eq:first-corollary}
\mat{U}_B = \mat{U}_A \mat{S}.
\end{equation}
\end{corollary}
\begin{proof}
Since the minimum of (\ref{Sharma_eq:frobenius-difference}) is achieved for $\mat{X}=\mat{I}$ and since the entries of $\mat{X}$ are $z_{ij}^2$, we have that $z_{ii}=\pm 1$, which corresponds to $\mat{Z}=\mat{S}$.  $\quad\blacksquare$
\end{proof}
\begin{corollary}
\label{corollary:hoffman-wielandt}
If $\mat{Q}$ is an orthogonal matrix, then
\begin{equation}\label{Sharma_eq:corollary-Wielandt-Hoffman}
\sum_{i=1}^{n}(\alpha_i-\beta_i)^2 \leq \frobenius{\adjmat_A - \mat{Q}\adjmat_B\mat{Q}\tp}^2.
\end{equation}
\end{corollary}
\begin{proof}
Since the eigen-decomposition of matrix $\mat{Q}\adjmat_B\mat{Q}\tp$ is $(\mat{Q}\mat{U}_B)\mat{\mat{\Lambda}}_B(\mat{Q}\mat{U}_B)\tp$ and since it
has the same eigenvalues as  $\adjmat_B$, the inequality (\ref{Sharma_eq:corollary-Wielandt-Hoffman}) holds and hence corollary~\ref{corollary:hoffman-wielandt}. $\quad\blacksquare$
\end{proof}

These corollaries will be useful in the case of spectral graph matching methods presented below. 

\subsection{Umeyama's Method}
\label{Sharma_subsection:UmeyamaMethod}
The exact spectral matching solution presented in section~\ref{Sharma_subsection:exact-solution} 
finds a permutation matrix satisfying (\ref{Sharma_eq:exact-perm}). This requires an exhaustive 
search over the space of all possible $2^n$ matrices. Umeyama's method presented in ~\cite{Umeyama88} 
proposes a relaxed solution to this problem as outlined below. 

Umeyama~\cite{Umeyama88} addresses the problem of 
\textit{weighted graph matching}\index{weighted graph matching} within the 
framework of spectral graph theory. He proposes two methods, the first
for \textit{undirected weighted graphs} and the second for
\textit{directed weighted graphs}. The adjacency matrix is used in
both cases. Let's consider the case of undirected graphs.
The eigenvalues are (possibly with multiplicities):
\begin{eqnarray}
\label{Sharma_eq:order-alpha}
\adjmat_A: & \alpha_1 \leq \hdots \leq \alpha_i \leq \hdots \leq \alpha_n\\
\label{Sharma_eq:order-beta}
\adjmat_B: & \beta_1 \leq \hdots \leq \beta_i \leq \hdots \leq \beta_n.
\end{eqnarray}
%
%
%
%\subsection{The Umeyama theorem}
%\label{Sharma_subsection:UmeyamaTheorem}
\begin{theorem}
\label{Sharma_theorem:umeyama}
(Umeyama)
If $\adjmat_A$ and $\adjmat_B$ are real-symmetric matrices with $n$ distinct eigenvalues (that can be ordered),
$\alpha_1 < \hdots < \alpha_i < \hdots < \alpha_n$ and $\beta_1 < \hdots < \beta_i < \hdots < \beta_n$, the minimum of :
\begin{equation}\label{Sharma_eq:Umeyama-minimizer-function}
J(\mat{Q}) = \frobenius{\adjmat_A - \mat{Q}\adjmat_B\mat{Q}\tp}^2
\end{equation}
is achieved for:
\begin{equation}\label{Sharma_eq:approx-isomorphism}
\mat{Q}^{\star} =  \mat{U}_A \mat{S} \mat{U}_B\tp
\end{equation}
and hence (\ref{Sharma_eq:corollary-Wielandt-Hoffman}) becomes an equality:
\begin{equation}\label{Sharma_eq:umeyama}
\sum_{i=1}^{n}(\alpha_i-\beta_i)^2 = \frobenius{\adjmat_A - \mat{Q}^{\star}\adjmat_B{\mat{Q}^{\star}}\tp}^2.
\end{equation}
\end{theorem}
\begin{proof}
The proof is straightforward. By corollary \ref{corollary:hoffman-wielandt}, the Hoffman-Wielandt theorem applies to matrices $\adjmat_A$ and $ \mat{Q}\adjmat_B{\mat{Q}}\tp$. By corollary \ref{corollary:first-hoffman-wielandt}, the equality (\ref{Sharma_eq:umeyama}) is achieved for:
\begin{equation}
\mat{Z} = \mat{U}_A\tp\mat{Q}^{\star}\mat{U}_B = \mat{S}
\end{equation}
and hence (\ref{Sharma_eq:approx-isomorphism}) holds. $\quad\blacksquare$
\end{proof}
Notice that (\ref{Sharma_eq:approx-isomorphism}) can be written as:
\begin{equation}
\label{Sharma_eq:approx-isomorphism-more}
\mat{U}_A = \mat{Q}^{\star}  \mat{U}_B \mat{S}
\end{equation}
which is a \textit{relaxed} version of
(\ref{Sharma_eq:exact-alignment}): The permutation matrix in the exact
isomorphism case is replaced by an orthogonal matrix. 

\paragraph{A Heuristic for Spectral Graph Matching:}\index{spectral graph matching}
%\subsection{A Heuristic for Spectral Graph Matching}
%\label{Sharma_subsection:heuristicSpectralGraphMatching}

Let us consider again the exact solution outlined in
section~\ref{Sharma_subsection:exact-solution}. Umeyama suggests a heuristic
in order to avoid exhaustive search over all possible $2^n$ matrices
that satisfy (\ref{Sharma_eq:exact-perm}). One may easily notice that:
\begin{equation}
\frobenius{\mat{P} - \mat{U}_A\mat{S}\mat{U}_B\tp}^2 = 2n - 2
\trace(\mat{U}_A\mat{S}(\mat{P}\mat{U}_B)\tp).
\label{Sharma_eq:norm-trace}
\end{equation}
Using Umeyama's notations, $\mat{\bar{U}}_A=[|u_{ij}|],\mat{\bar{U}}_B=[|v_{ij}|]$ (the entries of $\mat{\bar{U}}_A$ are the absolute values of the entries of $\mat{U}_A$), one may further notice that:
\begin{equation}
\trace(\mat{U}_A\mat{S}(\mat{P}\mat{U}_B)\tp) = \sum_{i=1}^n\sum_{j=1}^n s_j
u_{ij} v_{\pi(i)j} \leq \sum_{i=1}^n\sum_{j=1}^n |u_{ij}|| 
v_{\pi(i)j}| = \trace(\mat{\bar{U}}_A\mat{\bar{U}}_B\tp\mat{P}\tp).
\label{Sharma_eq:Umeyama-heuristic}
\end{equation}

The minimization of (\ref{Sharma_eq:norm-trace}) is equivalent to the
maximization of (\ref{Sharma_eq:Umeyama-heuristic}) and the maximal value
that can be attained by the latter is $n$. Using the fact that both
$\mat{U}_A$ and $\mat{U}_B$ are orthogonal matrices, one can easily conclude
that:
\begin{equation}
 \trace(\mat{\bar{U}}_A\mat{\bar{U}}_B\tp\mat{P}\tp) \leq n.
\label{Sharma_eq:check-permutation}
\end{equation}
Umeyama concludes that when the two graphs are isomorphic, the optimum
permutation matrix maximizes
$\trace(\mat{\bar{U}}_A\mat{\bar{U}}_B\tp\mat{P}\tp)$ and this can be solved by the
Hungarian algorithm~\cite{burkard2009}.\index{Hungarian algorithm}

When the two graphs are not exactly isomorphic,
theorem~\ref{Sharma_theorem:hoffman-wielandt} and
theorem~\ref{Sharma_theorem:umeyama} allow us to relax the permutation matrices
to the group of orthogonal matrices. Therefore with similar arguments
as above we obtain:
\begin{equation}
\trace(\mat{U}_A\mat{S}\mat{U}_B\tp\mat{Q}\tp) \leq \trace(\mat{\bar{U}}_A\mat{\bar{U}}_B\tp\mat{Q}\tp) \leq n.
\label{Sharma_eq:check-orthogonal}
\end{equation}
The permutation matrix obtained with the Hungarian algorithm can be
used as an initial solution that can then be improved by some
hill-climbing or relaxation technique \cite{Umeyama88}.

The spectral matching solution presented in this section is not directly applicable to large graphs.
In the next section we introduce the notion of dimensionality reduction for graphs which will lead
to a tractable graph matching solution. 
        % Section 3 
\section{Graph Embedding and Dimensionality Reduction}
\label{Sharma_section:GraphMatchingwithDimReduction}
%!!! please use unique labels (e.g., include your initials) for all
%your sections, equations, figures, tables, etc.
%
%In this section we will introduce Laplacian embeddings and graph PCA. 
%
%list problem's with Umeyama's approach
For \textit{large} and \textit{sparse} graphs, the results of
section~\ref{Sharma_section:SpectralGraphIsomorphism} and Umeyama's method (section~\ref{Sharma_subsection:UmeyamaMethod}) hold
only \textit{weakly}. Indeed, one cannot guarantee that all the eigenvalues
have multiplicity equal to one:
the presence of symmetries causes some of
eigenvalues to have an algebraic multiplicity greater than
one. Under these circumstances and due to numerical approximations,
it might not be possible to properly order the
eigenvalues. Moreover, for very large graphs with thousands of
vertices it is not practical to compute all its eigenvalue-eigenvector
pairs. This means that one has to devise a method that is able to
match shapes using a small set of eigenvalues and eigenvectors.

One elegant way to overcome this problem, is to reduce
the dimension of the eigenspace, along the line of spectral
dimensionality reductions techniques. The eigendecomposition 
of graph Laplacian matrices (introduced in section~\ref{Sharma_subsection:Laplacians}) 
is a popular choice for the purpose of dimensionality reduction~\cite{belkin2003laplacian}.
%Cite Laplacian papers
%%
\subsection{Spectral Properties of the Graph Laplacian}
\label{Sharma_subsection:LaplacianSpectralProperties}

The spectral properties of the Laplacian matrices introduced in section~\ref{Sharma_subsection:Laplacians} 
have been thoroughly studied. They are summarized in table~\ref{Sharma_table:Laplacian-summary}. 
\begin{table*}[b!h!t!]
\begin{center}
\begin{tabular}{|l|l|l|l|}
\hline
Laplacian & Null space & Eigenvalues & Eigenvectors \\
\hline
$\lapmat=\mat{U}\mat{\Lambda}\mat{U}\tp$ & $\vec{u}_1=\mathbbm{1}$ &
$0=\lambda_1 < \lambda_2 \leq \hdots \leq \lambda_n $
& $\vec{u}_{i>1}\tp\mathbbm{1}=0, \vec{u}_{i}\tp\vec{u}_{j}=\delta_{ij}$\\
\hline
$\normlapmat=\tilde{\mat{U}}\mat{\Gamma}\tilde{\mat{U}}\tp$ & $\tilde{\vec{u}}_1=\degmat^{1/2}\mathbbm{1}$ &
$0=\gamma_1 < \gamma_2 \leq \hdots \leq \gamma_n$
& $\tilde{\vec{u}}_{i>1}\tp\degmat^{1/2}\mathbbm{1}=0, \tilde{\vec{u}}_{i}\tp\tilde{\vec{u}}_{j}=\delta_{ij}$\\
\hline
$\normlapmat_R=\mat{T}\mat{\Gamma}\mat{T}\inverse$,
$\mat{T}=\mat{D}^{-1/2}\tilde{\mat{U}}$ & $\vec{t}_1=\mathbbm{1}$ & 
$0=\gamma_1 < \gamma_2 \leq \hdots \leq \gamma_n$
& $\vec{t}_{i>1}\tp\degmat\mathbbm{1}=0, \vec{t}_{i}\tp\degmat\vec{t}_{j}=\delta_{ij}$\\
\hline
\end{tabular}
\end{center}
\caption{Summary of the spectral properties of the Laplacian
  matrices. Assuming a connected graph, the null eigenvalue
  ($\lambda_1,\gamma_1$) has
  multiplicity one. The first non null eigenvalue ($\lambda_2,\gamma_2$) is
  known as the Fiedler value\index{Fiedler value} and its multiplicity is, in general, equal
  to one. The associated eigenvector is denoted the Fiedler vector~\cite{Chung97}.\index{Fiedler vector}}
\label{Sharma_table:Laplacian-summary}
\end{table*}
We derive some subtle properties of the
combinatorial Laplacian\indexsub{combinatorial}{graph Laplacian} which will be
useful for the task of shape registration. In particular, we will show
that the eigenvectors of the combinatorial Laplacian can be
interpreted as directions of maximum variance (principal components)
of the associated embedded shape representation. We note that the
embeddings of the 
normalized and random-walk Laplacians have different spectral
properties which make them less interesting for shape registration,
i.e., Appendix~\ref{Sharma_appendix:Properties-normalized-laplacian}.

\paragraph*{The combinatorial Laplacian.}\indexsub{combinatorial}{graph Laplacian} Let $\lapmat=\mat{U}\mat{\Lambda}\mat{U}\tp$ be the spectral 
decomposition of the combinatorial Laplacian with
$\mat{U}\mat{U}\tp=\mat{I}$. Let $\mat{U}$ be written as:
\begin{equation}
\mat{U} = \left[
\begin{array}{ccccc}
u_{11} & \hdots & u_{1k} & \hdots & u_{1n} \\
\vdots &        & \vdots &        & \vdots \\
u_{n1} & \hdots & u_{nk} & \hdots & u_{nn}
\end{array} \right]
\end{equation}
Each column of $\mat{U}$, $\vec{u}_k= (u_{1k}\hdots
u_{ik} \hdots u_{nk})\tp$ is an eigenvector associated with the
eigenvalue $\lambda_k$. From the definition of $\lapmat$ in (\ref{Sharma_eq:L-combinatorial}) (see~\cite{belkin2003laplacian})
one can easily see that $\lambda_1=0$ and that $\vec{u}_1=\mathbbm{1}$ (a constant vector).
Hence, $\vec{u}_{k\geq 2}\tp\mathbbm{1}=0$ and by combining this with
$\vec{u}_{k}\tp\vec{u}_{k}=1$, we derive the following proposition:
\begin{proposition}
\label{Sharma_prop:eigenvector-constraints}
The components of the non-constant eigenvectors of the combinatorial Laplacian satisfy the following constraints:
\begin{eqnarray}
\label{Sharma_eq:sum-of-entries-1}
&\sum_{i=1}^n u_{ik} = 0,&  \forall k, 2 \leq k \leq n\\
\label{Sharma_eq:limits-of-entries-1}
&-1 < u_{ik} < 1,& \forall i,k, 1\leq i \leq n, 2 \leq k \leq n.
\end{eqnarray}
\end{proposition}
Assuming
a connected graph, $\lambda_1$ has multiplicity equal to one
\cite{Luxburg2007}. Let's organize the eigenvalues of $\lapmat$ in
increasing order: $0 = \lambda_1 <
\lambda_2 \leq \hdots \leq \lambda_n$. We prove the following
proposition \cite{Chung97}:
\begin{proposition}
\label{Sharma_prop:max-lambda}
For all $k\leq n$, we have $\lambda_k \leq 2 \max_i (d_i)$, where
$d_i$ is the degree of vertex $i$.
\end{proposition}
\begin{proof}
The largest eigenvalue of $\lapmat$ corresponds to the maximization of
the Rayleigh quotient, or
\begin{equation}
\lambda_n=\max_{\vec{u}}
\frac{\vec{u}\tp\lapmat\vec{u}}{\vec{u}\tp\vec{u}}.
\end{equation}
We have $\vec{u}\tp\lapmat\vec{u}=\sum_{e_{ij}}w_{ij}(u_i-u_j)^2$. From
  the inequality $(a-b)^2\leq 2(a^2+b^2)$ we obtain:
\begin{equation}
\lambda_n \leq \frac {2\sum_{e_{ij}}w_{ij}(u_i^2+u_j^2)}{\sum_i u_i^2}
  = \frac {2\sum_i d_i u_i^2}{\sum_i u_i^2} \leq 2\max_i(d_i). 
\quad\blacksquare
\end{equation}
\end{proof}

This ensures an upper limit on the eigenvalues of $\lapmat$. 
%The spectral decomposition of the combinatorial Laplacian is
%$\lapmat=\mat{U}\mat{\Lambda}\mat{U}\tp$ 
By omitting the zero eigenvalue and associated eigenvector, we can rewrite $\lapmat$ as:
\begin{equation}
\lapmat = \sum_{k=2}^{n} \lambda_k \vec{u}_k \vec{u}_k\tp.
\end{equation}
Each entry $u_{ik}$ of an eigenvector $\vec{u}_k$ can be interpreted as
a real-valued function that projects a graph vertex $v_i$ onto that
vector. The mean and variance of the set $\{u_{ik}\}_{i=1}^{n}$ are
therefore a
measure of how the graph \textit{spreads} when projected onto the
$k$-th eigenvector. This is clarified by the following result:
\begin{proposition}
\label{Sharma_prop:spread-of-u}
The \textit{mean} $\overline{u}_k$ and
the \textit{variance} $\sigma_{u_k}$ of an eigenvector $\vec{u}_k$. For $2 \leq k \leq
n,$ and $1\leq i \leq n$ we have
\begin{eqnarray}
\label{Sharma_eq:sum-of-entries-1}
&\overline{u}_k  = & \sum_{i=1}^n u_{ik} = 0\\
\label{Sharma_eq:variance-1}
&\sigma_{u_k}  = & \frac{1}{n} \sum_{i=1}^{n} (u_{ik} -
\overline{u}_k)^{2} = \frac{1}{n} 
%\\
%\label{Sharma_eq:limits-of-entries-1}
%& |u_{ik}| <& 1
\end{eqnarray}
\end{proposition}
\begin{proof}
These results can be easily obtained from
$\vec{u}_{k\geq 2}\tp\mathbbm{1}=0$ and
$\vec{u}_{k}\tp\vec{u}_{k}=1$.  $\quad\blacksquare$
\end{proof}
These properties will be useful while aligning two Laplacian embeddings and thus registering two 3D shapes.

\subsection{Principal Component Analysis of a Graph Embedding}
\label{Sharma_subsection:GraphEmbeddingAndPCA}
% Introduce gram matrix notation for pseudo inverse of laplacian and derive expression for CTD embedding.
% relate this to Green's function

The Moore-Penrose pseudo-inverse of the Laplacian can be written as:
\begin{eqnarray}
\lapmat^{\dag} &=& \mat{U} \mat{\Lambda}^{-1} \mat{U}\tp \nonumber \\ 
&=& (\mat{\Lambda}^{-\frac{1}{2}}\mat{U}\tp)\tp (\mat{\Lambda}^{-\frac{1}{2}}\mat{U}\tp)  \nonumber \\
&=& \mat{X}\tp\mat{X}
\label{Sharma_eq:Gram-matrix}
\end{eqnarray}
where $\mat{\Lambda}\inverse = \diag(0, 1/\lambda_2,\hdots,1/\lambda_n)$. 
%is a $(n-1) \times (n-1)$ matrix and 
%similarly $\mat{U} = [\vec{u}_2,\hdots,\vec{u}_n]$ is a $n \times (n-1)$ matrix after truncating zero
%eigenvalue and corresponding eigenvector. 

The symmetric semi-definite positive matrix $\lapmat^{\dag}$ is a \textit{Gram} matrix\index{Gram matrix}
with the same eigenvectors as those of the graph Laplacian. When omitting the null eigenvalue and associated constant eigenvector,  $\mat{X}$ becomes
a $(n-1) \times n$ matrix whose columns are the coordinates of the
graph's vertices in an \textit{embedded (or feature) space}, i.e.,
$\mat{X}=[\vec{x}_1\hdots \vec{x}_j\hdots\vec{x}_n]$. It is interesting to note that 
the entries of $\lapmat^{\dag}$ may be viewed as \textit{kernel}
dot-products, or a Gram matrix \cite{Ham-Lee-ICML-2004}. 
The Gram-matrix representation allows us to embed the graph in an Euclidean feature-space 
where each vertex $v_j$ of the graph is a feature point represented as $\vec{x}_j$.

The left pseudo-inverse operator of the Laplacian $\lapmat$, satisfying $\lapmat^{\dag} \lapmat \vec{u} = \vec{u}$ 
for any $\vec{u} \bot \mathrm{null}(\lapmat)$, is also called the \emph{Green function} of the heat equation.
Under the assumption that the graph is connected and thus $\lapmat$ has an
eigenvalue $\lambda_1 = 0$ with multiplicity 1, we obtain:
\begin{equation}
\label{Sharma_eq:GreensFunction}
\lapmat^{\dag} = \sum_{k=2}^{n} \frac{1}{\lambda_k} \vec{u}_k \vec{u}_k\tp.
\end{equation}
The Green function is 
intimately related to random walks on graphs, and can be interpreted probabilistically as follows.
Given a Markov chain such that each graph vertex is the state, and the transition 
from vertex $v_i$ is possible to any adjacent
vertex $v_j \sim v_i$ with probability $w_{ij}/d_i$, the expected
number of steps required to reach vertex $v_j$ from $v_i$, called the
\emph{access} or \emph{hitting time} $O(v_i,v_j)$. 
The expected number of steps in a round trip from $v_i$ to $v_j$ is
called the \emph{commute-time distance}\index{commute-time distance}: $\CTD^2(v_i,v_j)=O(v_i,v_j)+O(v_j,v_i)$.
The commute-time distance \cite{QiuHancock2007PAMI} can be expressed in terms of the entries of $\lapmat^{\dag}$:
\begin{eqnarray}
\label{Sharma_eq:CTD-closedform}
\CTD^2(v_i,v_j) &=&  \vol(\G) (\lapmat^{\dag}(i,i) + \lapmat^{\dag}(j,j) - 2\lapmat^{\dag}(i,j)) \nonumber \\
&=& \vol(\G) \left(\sum_{k=2}^n \frac{1}{\lambda_k} \vec{u}_{ik}^2 +
  \sum_{k=2}^n \frac{1}{\lambda_k} \vec{u}_{jk}^2 - 2 \sum_{k=2}^n \frac{1}{\lambda_k} \vec{u}_{ik}\vec{u}_{jk}\right) \nonumber \\
&=& \vol(\G) \sum_{k=2}^n \left( \lambda_k^{-1/2}(\vec{u}_{ik} - \vec{u}_{jk})
\right)^2 \nonumber \\
&=& \vol(\G)  \|\vec{x}_i - \vec{x}_j \|^2, 
\end{eqnarray}
where the volume of the graph, $\vol(\G)$ is the sum of the degrees of
all the graph vertices. 
The $\CTD$ function is positive-definite and sub-additive, thus defining
a \emph{metric} between the graph vertices, referred to as
\emph{commute-time} (or \emph{resistance}) \emph{distance}
\cite{GrinsteadSnell98}.
The CTD is inversely related to the number and length of paths connecting two vertices. 
Unlike the shortest-path (geodesic) distance, CTD
captures the connectivity structure of the graph volume rather
than a single path between the two vertices. The great advantage of
the commute-time distance over the shortest geodesic path is that it
is robust to topological changes and therefore is well suited for
characterizing complex shapes.
Since the volume is a graph constant, we obtain:
\begin{equation}
\CTD^2(v_i,v_j) \propto  \|\vec{x}_i - \vec{x}_j \|^2.
\label{Sharma_eq:CTD-closedform1}
\end{equation}
Hence, the Euclidean distance between any two feature points $\vec{x}_i$ and $\vec{x}_j$ is the commute time distance
between the graph vertex $v_i$ and $v_j$.  

%Add proper explaination of dimensionality reduction.  
Using the first $K$ non-null eigenvalue-eigenvector pairs of the Laplacian $\lapmat$, the \textit{commute-time embedding}\index{commute-time embedding} of the graph's 
nodes corresponds to the column vectors of the $K\times n$ matrix $\mat{X}$:
\begin{equation}
\label{Sharma_eq:X-coordinates}
\mat{X}_{K\times n}=\mat{\Lambda}^{-1/2}_K(\mat{U}_{n\times K})\tp = [\vec{x}_1\hdots
\vec{x}_j\hdots\vec{x}_n].
\end{equation}
From (\ref{Sharma_eq:limits-of-entries-1}) and (\ref{Sharma_eq:X-coordinates}) one
can easily infer lower and upper bounds for the $i$-th coordinate of
$\vec{x}_j$:
\begin{equation}
-\lambda_i^{-1/2} < x_{ji} < \lambda_i^{-1/2}.
\label{Sharma_eq:xcoordinates-bounds}
\end{equation}
The last equation implies that the
graph embedding stretches along the eigenvectors with a factor that is
inversely proportional to the square root of the
eigenvalues. Theorem~\ref{Sharma_prop:graph-pca} below characterizes
the smallest non-null $K$ eigenvalue-eigenvector pairs of 
$\lapmat$  as the directions of maximum
variance (the principal components) of the commute-time embedding.
%This will be used (see
%section~\ref{eq:shape-alignment} below) to register two graphs using
%the alignment of their principal directions.

%\textbf{Proposition.}
\begin{theorem}
\label{Sharma_prop:graph-pca}
The largest eigenvalue-eigenvector pairs of the pseudo-inverse of the combinatorial Laplacian matrix are the principal components of the commute-time embedding, i.e., the points
$\mat{X}$ are zero-centered and have a diagonal covariance matrix.
\end{theorem}

\begin{proof}
%\textit{Proof:} 
Indeed, from
(\ref{Sharma_eq:sum-of-entries-1}) we obtain a zero-mean while from
(\ref{Sharma_eq:X-coordinates}) we obtain a diagonal covariance matrix:
\begin{equation}
\overline{\vec{x}} = \frac{1}{n} \sum_{i=1}^{n} \vec{x}_i =
\frac{1}{n} \mat{\Lambda}^{-\frac{1}{2}}
\left( 
\begin{array}{c}
\sum_{i=1}^{n}\vec{u}_{i2} \\ \vdots \\ \sum_{i=1}^{n}\vec{u}_{ik+1}
\end{array} 
\right) = \left( 
\begin{array}{c} 0 \\ \vdots \\ 0 \end{array} 
\right)
\end{equation}
\begin{equation}
\mat{\Sigma}_X = 
%\frac{1}{n} \sum_{i=1}^{n} \vv{x}_i \vv{x}_i\tp = 
\frac{1}{n}
\mat{X}\mat{X}\tp =  \frac{1}{n}
\mat{\Lambda}^{-\frac{1}{2}}\mat{U}\tp\mat{U} \mat{\Lambda}^{-\frac{1}{2}} =
\frac{1}{n} \mat{\Lambda}^{-1}
\end{equation}
 $\quad\blacksquare$.
\end{proof}
Figure~\ref{Sharma_fig:Eigenvector_projection} shows the projection of graph (in this case 3D shape
represented as meshes) vertices on eigenvectors. 
\begin{figure}[ht]
\begin{center}
\begin{tabular}{ccc}
\includegraphics[width=0.32\linewidth]{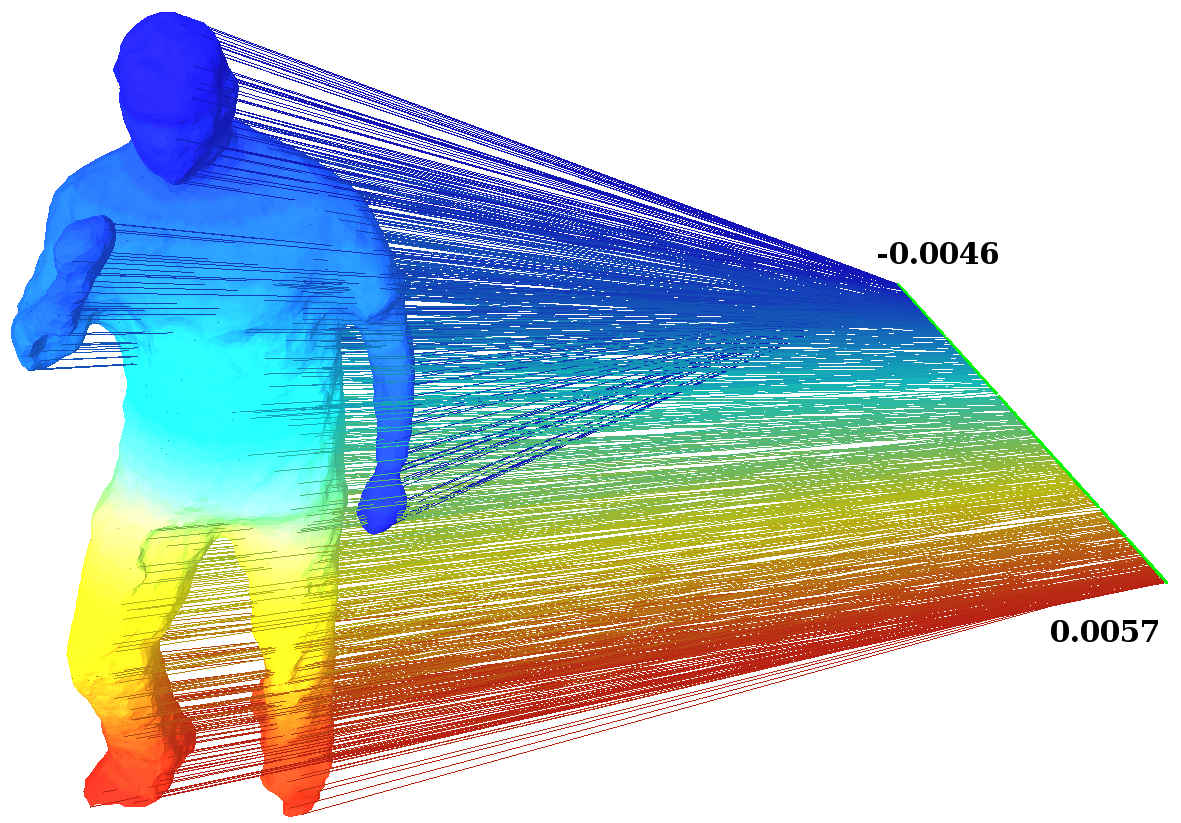} &
\includegraphics[width=0.32\linewidth]{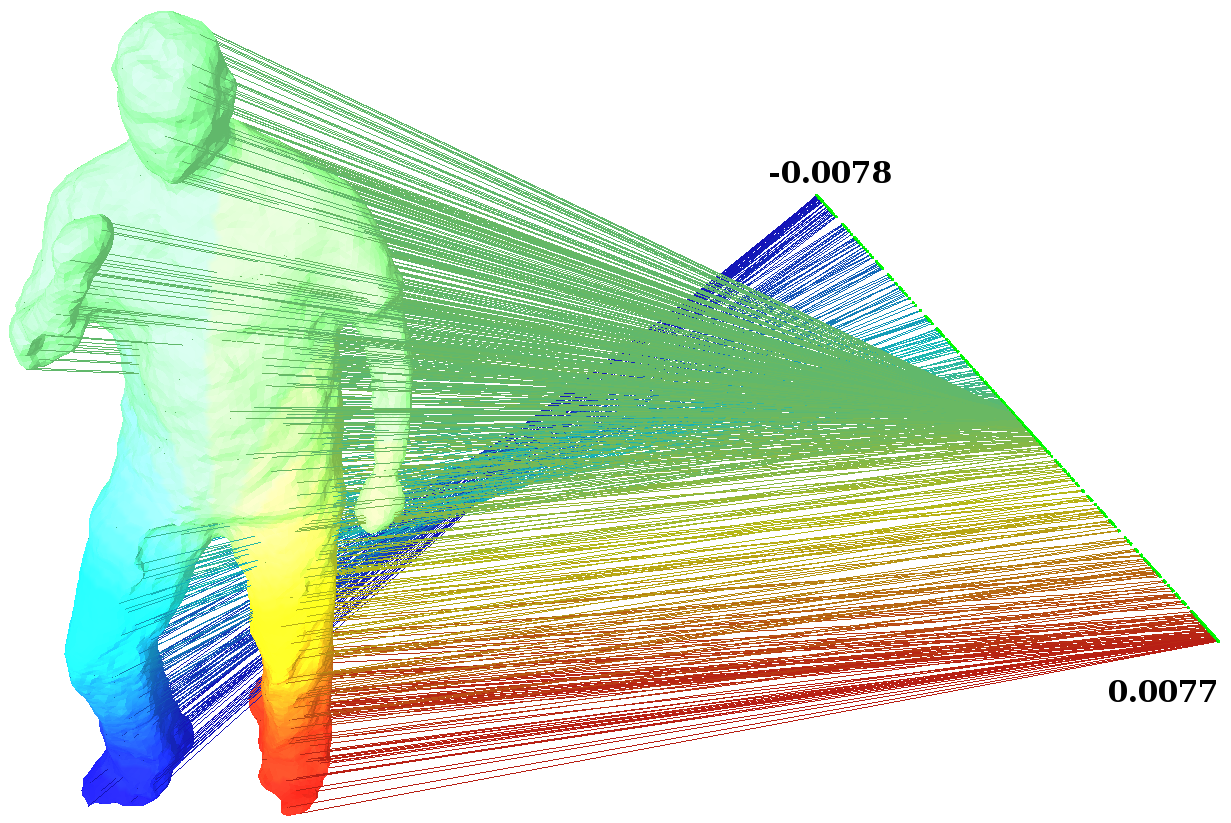} &
\includegraphics[width=0.32\linewidth]{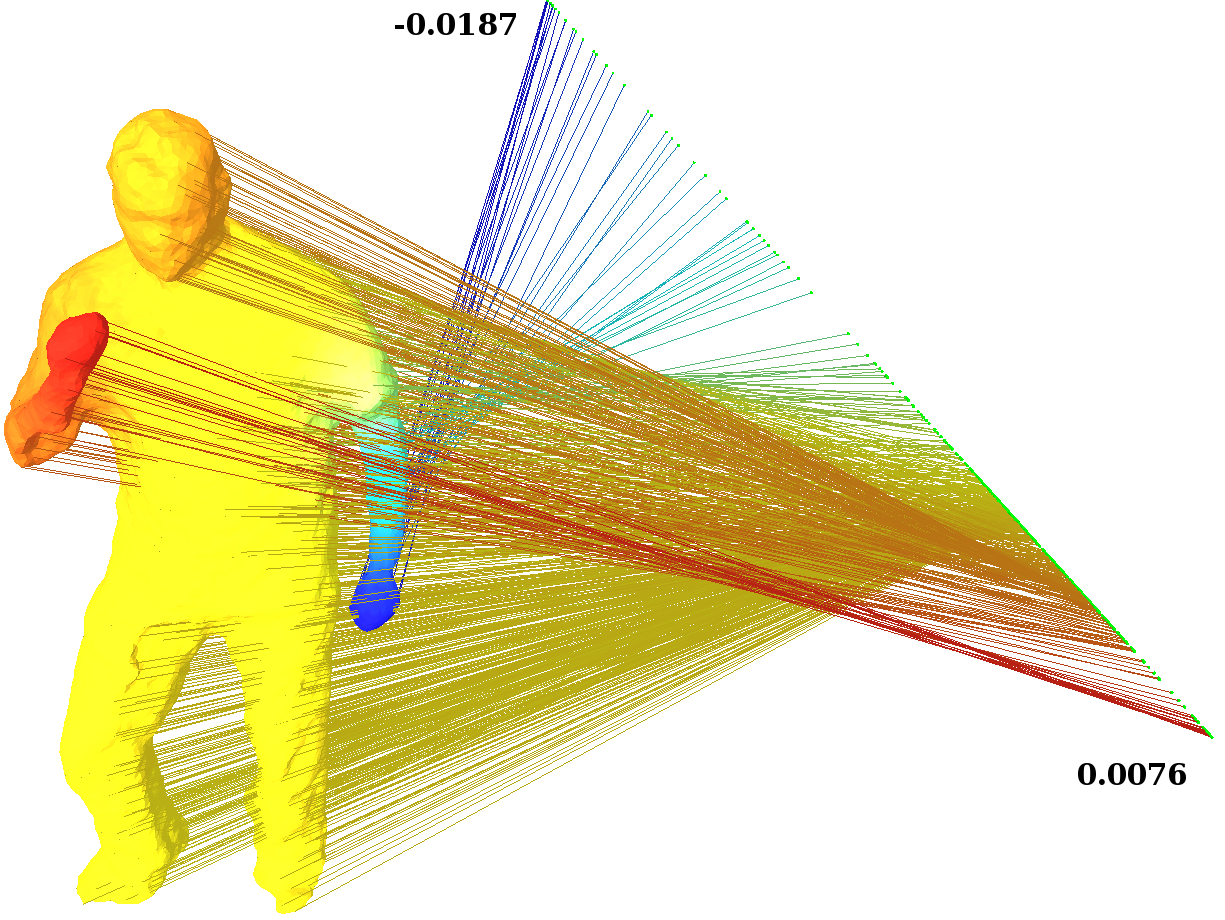} \\
(a) & (b) & (c)
\end{tabular}
\caption{This is an illustration of the concept of the PCA of a
  graph embedding. The graph's vertices are projected onto the second, third and
  fourth eigenvectors of the Laplacian matrix. These eigenvectors can be viewed as the principal directions of the shape.}
\label{Sharma_fig:Eigenvector_projection}
\end{center}
\end{figure}

%Thus, dimensionality reduction could be performed by choosing the
%smallest $K$ 
%eigenvectors of graph Laplacian to obtain a graph embedding which represent the graph in a 
%feature-space where each graph vertex is represented as a $K$-dimensional feature point 
%and the Euclidean distance in this feature-space captures the average connectivity 
%between two graph vertices. 
%
\subsection{Choosing the Dimension of the Embedding}
\label{Sharma_subsection:ChoosingEmbDim}

A direct consequence of theorem~\ref{Sharma_prop:graph-pca} 
is that the embedded graph representation is centered and the eigenvectors of the 
combinatorial Laplacian are the directions of maximum variance. The \textit{principal}
eigenvectors correspond to the eigenvectors associated with the $K$
\textit{largest} eigenvalues of the  $\lapmat^{\dag}$, i.e.,
$\lambda_2^{-1}\geq \lambda_3^{-1}\geq\hdots \geq \lambda_K^{-1}$. The variance along
vector $\vec{u}_k$ is $\lambda_k^{-1}/n$. Therefore, the total
variance can be computed from the trace of the $\lapmat^{\dag}$ matrix :
\begin{equation}
\trace(\mat{\Sigma}_X) = \frac{1}{n} \trace(\lapmat^{\dag}).
\end{equation}
A standard way of choosing the principal components is to use the
\textit{scree diagram}:
\begin{equation}
\label{Sharma_eq:scree-diagram}
\theta (K) = \frac{\sum_{k=2}^{K+1} \lambda_k^{-1}}{\sum_{k=2}^{n} \lambda_k^{-1}}.
\end{equation}
The selection of the first $K$ principal eigenvectors therefore
depends on the spectral fall-off of the inverses of the eigenvalues. 
%In theory, there is an orthonormal eigen-basis of $K$ vectors associated with
%$\lambda_2$. In practice, the eigenvalues computed by numerical
%eigen solvers do not necessarily converge to a particular multiple
%eigenvalue. Instead they converge to $K$ \textit{distinct} eigenvalues
%and associated eigenvectors that form an orthonormal basis. 
In spectral graph theory, the dimension $K$ is chosen on the basis of the existence of an 
%choosing the dimension $K$ is somehow equivalent to finding an
eigengap, such that $\lambda_{K+2} - \lambda_{K+1} > t$ with $t>0$.
In practice it is extremely difficult to find such an eigengap, in
particular in the case of sparse graphs that correspond to a
discretized manifold. Instead, we propose to select the dimension of
the embedding in the following way. Notice that
(\ref{Sharma_eq:scree-diagram}) can be written as $\theta (K)=A/(A+B)$ with
$A=\sum_{k=2}^{K+1} \lambda_k^{-1}$ and $B=\sum_{k=K+2}^{n}
 \lambda_k^{-1}$. Moreover, from the fact that the $\lambda_k$'s are
arranged in increasing order, we obtain $B\leq
(n-K-1) \lambda_{K+1}^{-1}$. Hence:
\begin{equation}
\theta_{\min} \leq \theta(K) \leq 1,
\end{equation}
with
\begin{equation}
\label{Sharma_eq:choosing_K}
\theta_{\min} = \frac{\sum_{k=2}^{K+1}  \lambda_k^{-1}}{\sum_{k=2}^{K}
   \lambda_k^{-1} + (n-K) \lambda_{K+1}^{-1}}.
\end{equation}
This lower bound can be computed from the $K$ smallest non null eigenvalues
of the combinatorial Laplacian matrix. Hence, one can choose $K$ such
that the sum of the first $K$ eigenvalues of the  $\lapmat^{\dag}$ matrix is
a good approximation of the total variance, e.g.,
$\theta_{\min}=0.95$.

\subsection{Unit Hyper-sphere Normalization}
\label{Sharma_subsection:hypersphere-normalization}
One disadvantage of the standard embeddings is that,
when two shapes have large difference in sampling the embeddings
will differ by a significant scale factor. 
In order to avoid this we can
\textit{normalize} the embedding such that the 
vertex coordinates lie on a unit sphere of dimension $K$, which yields:
\begin{equation}
\label{Sharma_eq:normalized_embedding}
\hat{\vec{x}}_i = \frac{\vec{x}_i}{\| \vec{x}_i \|}.
\end{equation}
In more detail, the $k$-th coordinate of $\hat{\vec{x}}_i$ writes as:
\begin{equation}
\hat{\vec{x}}_{ik} = \frac{\lambda_k^{-\frac{1}{2}} \vec{u}_{ik}}
{\left(\sum_{l=2}^{K+1}\lambda_{l}^{-\frac{1}{2}}\vec{u}_{il}^2\right)^{1/2}}.
\end{equation}
%The commute-time distance is equivalent to the geodesic distance on a
%unit hyper-sphere:
%\begin{eqnarray}
%\label{Sharma_eq:hypersphere-distance}
%d_{\mathcal{S}}(i,j) &=& \arccos \frac{\mat{G}(i,j)}{\mat{G}(i,i)^{1/2} \mat{G}(j,j)^{1/2}} \nonumber \\
%&=& \arccos \frac{\sum_{k=2}^{K+1} \lambda_k^{-1} \vec{u}_{ik}\vec{u}_{jk}}{(\sum_{k=2}^{K+1} \lambda_k^{-1} \vec{u}_{ik}^2)^{\frac{1}{2}} (\sum_{k=2}^{K+1} \lambda_k^{-1} \vec{u}_{jk}^2)^{\frac{1}{2}}}
%\end{eqnarray}
%
        % Section 4
\section{Spectral Shape Matching} % section title
\label{Sharma_section:AlignmentAndRegistration}  % \label{} allows reference to this section
%!!! please use unique labels (e.g., include your initials) for all
%your sections, equations, figures, tables, etc.
%
%In this section we will introduce the solution to maximum sub-graph isomorphism problem using 
%spectral graph matching methods in detail. 

In the previous sections we discussed solutions for the exact and
inexact graph isomorphism problems, we recalled the main spectral
properties of the combinatorial graph Laplacian, and we provided a novel analysis
of the commute-time embedding that allows to interpret the latter in
terms of the PCA of a graph, and to select the appropriate dimension
$K\ll n$ of the associated embedded metric space. In this section we
address the problem of 3D shape registration and we illustrate how the
material developed above can be exploited in order to build a robust
algorithm for spectral shape matching.  \indexsub{spectral}{shape matching}
% In the first step a histogram matching is employed to 
%align two shape/graph embeddings and in the second step Expectation-Maximization (EM) algorithm is used 
%for dense registration. 

Let's consider two shapes described by two graphs,  $\G_A$ and  $\G_B$
where $|\V_A| = n$ and $|\V_B| = m$.  Let $\lapmat_A$ and $\lapmat_B$
be their corresponding graph Laplacians. Without loss of generality,
one can choose the same dimension $K \ll \min (n,m)$ for the two embeddings.
This yields the following eigen
decompositions:
\begin{eqnarray}
\lapmat_A &=& \mat{U}_{n\times K} \mat{\Lambda}_K (\mat{U}_{n\times
  K})\tp \\
\lapmat_B &=& \mat{U}'_{m\times K} \mat{\Lambda}'_K (\mat{U}'_{m\times
  K})\tp.
\end{eqnarray}
For each one of these graphs, one can build two \textit{isomorphic} embedded representations, as
follows:
\begin{itemize}
\item An \textit{unnormalized Laplacian embedding} that uses the $K$
  rows of $\mat{U}_{n\times K}$ as the Euclidean coordinates of the
  vertices of $\G_A$ (as well as the $K$  rows of $\mat{U}'_{m\times K}$ as the Euclidean coordinates of the
  vertices of $\G_B$), and
\item A \textit{normalized commute-time embedding}\index{commute-time embedding} given by
  (\ref{Sharma_eq:normalized_embedding}), i.e., $\hat{\mat{X}}_A =
  [\hat{\vec{x}}_1\hdots \hat{\vec{x}}_j\hdots\hat{\vec{x}}_n]$ (as well as
$\hat{\mat{X}}_B = [\hat{\vec{x}}'_1\hdots \hat{\vec{x}}'_j\hdots
\hat{\vec{x}}'_m]$). We recall that each column $\hat{\vec{x}}_j$
(and respectively $\hat{\vec{x}}'_j$) is a $K$-dimensional vector
corresponding to a vertex $v_j$ of $\G_A$ (and respectively $v'_j$ of $\G_B$).
\end{itemize}

\subsection{Maximum Subgraph Matching and Point Registration}
\label{Sharma_subsection:MaximaSubgraphIsomorphism}

Let's apply the graph isomorphism framework of
Section~\ref{Sharma_section:SpectralGraphIsomorphism} to the two
graphs. They are embedded into two congruent spaces of dimension
$\mathbbm{R}^K$. If the smallest $K$ non-null eigenvalues associated with the
two embeddings are distinct and can be ordered, i.e.:
\begin{eqnarray}
\lambda_2 < \hdots < \lambda_k < \hdots < \lambda_{K+1}\\
\lambda'_2 < \hdots < \lambda'_k < \hdots < \lambda'_{K+1}
\end{eqnarray}
then, the Umeyama method
could be applied. If one uses the unnormalized Laplacian embeddings
just defined, (\ref{Sharma_eq:approx-isomorphism})
becomes: 
\begin{equation}\label{Sharma_eq:approx-isomorphism-Laplacian} 
\mat{Q}^{\star} = \mat{U}_{n\times K} \mat{S}_K (\mat{U}'_{m\times
  K})\tp
\end{equation}

Notice that here the sign matrix $\mat{S}$ defined in \ref{Sharma_eq:approx-isomorphism}  became a $K \times K$ matrix denoted by $\mat{S}_K$. We now assume that the eigenvalues $\{\lambda_2,\hdots,\lambda_{K+1}\}$ and $\{\lambda'_2,\hdots,\lambda'_{K+1}\}$ \textit{cannot be reliably ordered}. This can be modeled by multiplication with a $K \times K$ permutation matrix 
$\mat{P}_K$:
\begin{equation}\label{Sharma_eq:approx-isomorphism-Laplacian-with-permutation}
\mat{Q} =  \mat{U}_{n\times K}  \mat{S}_K  \mat{P}_K (\mat{U}'_{m\times
  K})\tp
\end{equation}
Pre-multiplication of $(\mat{U}'_{m\times
  K})\tp$ with $ \mat{P}_K$ permutes its rows such that $\vec{u}'_{k}
\rightarrow \vec{u}'_{\pi(k)}$. Each entry $q_{ij}$ of the $n\times m$ matrix
$\mat{Q}$ can therefore be written as:
\begin{equation}
\label{Sharma_eq:entry_Q_unnormalized}
%\mat{Q}^{\star} (i,j) 
q_{ij}= \sum_{k=2}^{K+1} s_k u_{ik} u'_{j\pi(k)}
\end{equation}
Since both $\mat{U}_{n\times K}$ and $\mat{U}'_{m\times K}$ are column-orthonormal matrices, the
dot-product defined by
(\ref{Sharma_eq:entry_Q_unnormalized}) is equivalent to the cosine of
the angle between two $K$-dimensional vectors. This means that each entry of
$\mat{Q}$ is such that $-1 \leq q_{ij} \leq +1$ and that two vertices
$v_i$ and $v'_j$ are matched if $q_{ij}$ is
close to 1. 

One can also use the normalized
commute-time coordinates and define an equivalent expression as above:
\begin{equation}
\label{Sharma_eq:Q_normalized}
\hat{\mat{Q}}=  \hat{\mat{X}}\tp \mat{S}_K \mat{P}_K \hat{\mat{X}}'
\end{equation}
with:
\begin{equation}
\label{Sharma_eq:entry_Q_normalized}
\hat{q}_{ij} = \sum_{k=2}^{K+1} s_k \hat{x}_{ik} \hat{x}'_{j\pi(k)}
\end{equation}
Because both sets of points $\hat{\mat{X}}$ and $\hat{\mat{X}}'$ lie on a $K$-dimensional unit hyper-sphere, we also have 
$-1 \leq \hat{q_{ij}} \leq +1$.

It should however be emphasized that the rank of the $n\times m$
matrices $\mat{Q}, \hat{\mat{Q}}$ is equal to $K$. Therefore, these
matrices cannot be viewed as \textit{relaxed permutation matrices}
between the two graphs. In fact they define many-to-many
correspondences between the vertices of the first graph and the
vertices of the second graph, this being due to the fact that the
graphs are embedded on a low-dimensional space. This is one of the
main differences between our method proposed in the next section and the Umeyama method, as well as
many other subsequent methods, that use all eigenvectors of the graph. 
As it will be explained below, our formulation leads to a shape
matching method that will alternate between aligning their eigenbases
and finding a vertex-to-vertex assignment. 

It is possible to extract a one-to-one assignment matrix
from $\mat{Q}$ (or from $\hat{\mat{Q}}$) using either
dynamic programming or an assignment method technique such as the
Hungarian algorithm. Notice that this assignment is conditioned by the
choice of a sign matrix $\mat{S}_K$ and of a permutation matrix
$\mat{P}_K$, i.e., $2^KK!$ possibilities, and that not all these
choices correspond to a valid sub-isomorphism between the two
graphs. Let's consider the case of the normalized commute-time
embedding; there is an equivalent formulation for the unnormalized
Laplacian embedding. The two graphs are described by two sets of
points, $\hat{\mat{X}}$ and $\hat{\mat{X}}'$, both lying onto the
$K$-dimensional unity hyper-sphere. The $K\times K$ matrix $\mat{S}_K
\mat{P}_K$ transforms one graph embedding onto the other graph
embedding. Hence, one can write $\hat{\vec{x}}_i = \mat{S}_K \mat{P}_K
\hat{\vec{x}}'_j$ if vertex $v_i$ matches $v_j$. More generally 
Let $ \mat{R}_K = \mat{S}_K \mat{P}_K$ and let's extend
the domain of  $\mat{R}_K$ to all possible orthogonal matrices\index{orthogonal matrices} of size $K \times K$, namely $\mat{R}_K\in \mathcal{O}_K$ or the \textit{orthogonal group} of dimension $K$. We can now write the following criterion whose minimization over $\mat{R}_K$ guarantees an optimal solution for registering the vertices of the first graph with the vertices of the second graph:
\begin{equation}
\label{Sharma_eq:orthogonal-transformation}
\min_{R_K} \sum_{i=1}^n\sum_{j=1}^m \hat{q}_{ij} \| \hat{\vec{x}}_i -  \mat{R}_K \hat{\vec{x}}_j' \|^2
\end{equation}

One way to solve minimization problems such as (\ref{Sharma_eq:orthogonal-transformation}) is to use a point registration algorithm that alternates
between (i)~estimating the $K \times K$ orthogonal transformation $\mat{R}_K$,
which aligns the $K$-dimensional coordinates associated with 
the two embeddings, and (ii)~updating the assignment variables $\hat{q}_{ij}$. This can be done using either ICP-like methods (the $\hat{q}_{ij}$'s are binary variables), or EM-like methods (the $\hat{q}_{ij}$'s are posterior probabilities of assignment variables). 
As we just outlined above, matrix $\mat{R}_K$ belongs to the orthogonal group $\mathcal{O}_K$.
Therefore this framework differs from standard implementations of ICP and EM algorithms that usually estimate a 2-D or 3-D \textit{rotation} matrix which belong to the \textit{special orthogonal group}.

It is well established that ICP algorithms are easily trapped in local minima. The EM algorithm recently proposed in \cite{HFYDZ11} is able to converge to a good solution starting with a rough initial guess and is robust to the presence of outliers. Nevertheless, the algorithm proposed in \cite{HFYDZ11} performs well under \textit{rigid transformations} (rotation and translation), whereas in our case we have to estimate a more general orthogonal transformation that incorporates both rotations and reflections. Therefore, before describing in detail an EM algorithm well suited for solving the problem at hand, we discuss the issue of estimating an initialization for the transformation aligning the $K$ eigenvectors of the first embedding with those of the second embedding and we propose a practical method for initializing this transformation (namely, matrices $\mat{S}_K$ and  $\mat{P}_K$ in (\ref{Sharma_eq:Q_normalized})) based on comparing the histograms of these eigenvectors, or \textit{eigensignatures}.

\subsection{Aligning Two Embeddings Based on Eigensignatures}
\label{Sharma_subsection:HistogramMatching}

Both the unnormalized Laplacian embedding and the normalized commute-time embedding of a graph are represented in a metric space spanned by the eigenvectors of the Laplacian matrix, namely the n-dimensional vectors $\{\vec{u}_2, \hdots, \vec{u}_k, \hdots, \vec{u}_{K+1}\}$, where $n$ is the number of graph vertices. They correspond to \textit{eigenfunctions} and each such eigenfunction maps the graph's vertices onto the real line. More precisely, the $k$-th eigenfunction maps a vertex $v_i$ onto $u_{ik}$. 
Propositions~\ref{Sharma_prop:eigenvector-constraints} and \ref{Sharma_prop:spread-of-u} revealed interesting statistics of the sets $\{u_{1k}, \hdots, u_{ik}, \hdots, u_{nk}\}_{k=2}^{K+1}$. Moreover, theorem~\ref{Sharma_prop:graph-pca} provided an interpretation of the eigenvectors in terms of principal directions of the embedded shape.  One can therefore conclude that the  probability distribution of the components of an eigenvector have interesting properties that make them suitable for comparing two shapes, namely $-1 < u_{ik} < +1$, $\overline{u}_k = 1/n\sum_{i=1}^n u_{ik} = 0$, and $\sigma_k = 1/n \sum_{i=1}^n u_{ik}^2 = 1/n$. This means that one can build a histogram for each eigenvector and that all these histograms share the same bin width $w$ and the same number of bins $b$ \cite{Scott1979}:
\begin{eqnarray}
w_k &=& \frac{3.5 \sigma_k}{n^{1/3}} = \frac{3.5}{n^{4/3}}\\
b_k &=& \frac{\sup_i u_{ik} - \inf_i u_{ik}}{w_k} \approx \frac{n^{4/3}}{2}.
\end{eqnarray}

We claim that these histograms are eigenvector signatures which are invariant under graph isomorphism. Indeed, let's consider the Laplacian $\lapmat$ of a shape and we apply the isomorphic transformation $\mat{P}\lapmat\mat{P}\tp$ to this shape, where $\mat{P}$ is a permutation matrix. If $\vec{u}$ is an eigenvector of $\lapmat$, it follows that $\mat{P}\vec{u}$ is an eigenvector of $\mat{P}\lapmat\mat{P}\tp$ and therefore, while the order of the components of $\vec{u}$ are affected by this transformation, their frequency  and hence their probability distribution remain the same. Hence, one may conclude that such a histogram may well be viewed as an \textit{eigensignature}. 

We denote with $H\{\vec{u}\}$ the histogram formed with the components
of $\vec{u}$ and let $C(H\{\vec{u}\},H\{\vec{u}'\})$ be a similarity
measure between two histograms. From the eigenvector properties just
outlined, it is straightforward to notice that $H\{\vec{u}\}\neq
H\{\vec{-u}\}$: These two histograms are mirror symmetric. Hence, the histogram is not invariant to the sign of an eigenvector. 
Therefore one can use the eigenvectors' histograms to estimate both the permutation matrix $\mat{P}_K$ and the sign matrix $\mat{S}_K$ in (\ref{Sharma_eq:Q_normalized}).
The problem of finding  one-to-one assignments $\{\vec{u}_k \leftrightarrow s_k \vec{u}'_{\pi(k)}\}_{k=2}^{K+1}$ between the two sets of eigenvectors associated with the two shapes is therefore equivalent to the problem of finding one-to-one assignments between their histograms. 

Let $\mat{A}_K$ be an assignment matrix between the histograms of the first shape and the histograms of the second shape. Each entry of this matrix is defined by:
\begin{equation}
\label{Sharma_eq:histo_matching_score}
a_{kl} = \sup [ C(H\{\vec{u}_k\},H\{\vec{u}_l'\}) ; C(H\{\vec{u}_k\},H\{\vec{-u}_l'\}) ]
\end{equation}
Similarly, we define a matrix $\mat{B}_K$ that accounts for the \textit{sign assignments}:
\begin{equation}
b_{kl} = \left\{
\begin{array}{ccc}
+1 & \mbox{if} & C(H\{\vec{u}_k\},H\{\vec{u}_l'\}) \geq C(H\{\vec{u}_k\},H\{\vec{-u}_l'\}) \\
-1 & \mbox{if} & C(H\{\vec{u}_k\},H\{\vec{u}_l'\}) < C(H\{\vec{u}_k\},H\{\vec{-u}_l'\})
\end{array}
\right.
\end{equation}

Extracting a permutation matrix $\mat{P}_K$ from $\mat{A}_K$ is an
instance of the bipartite maximum matching problem and the Hungarian
algorithm\index{Hungarian algorithm} is known to provide an optimal solution to this assignment
problem~\cite{burkard2009}. Moreover, one can use the estimated
$\mat{P}_K$ to extract a sign matrix  $\mat{S}_K$ from
$\mat{B}_K$. Algorithm~\ref{Sharma_algo:Align} estimates an alignment between two
embeddings.
\begin{algorithm}[h!]
\caption{\textit{Alignment of Two Laplacian Embeddings}}
\label{Sharma_algo:Align}
\begin{algorithmic}[1]
\INPUT : Histograms associated with eigenvectors
$\{\vec{u}_k\}_{k=2}^{K+1}$ and $\{\vec{u}'_k\}_{k=2}^{K+1}$.
\OUTPUT: A permutation matrix $\mat{P}_K$ and a sign matrix $\mat{S}_K$.
\STATE Compute the assignment matrices $\mat{A}_K$ and $\mat{B}_K$.
\STATE Compute $\mat{P}_K$ from $\mat{A}_K$ using the Hungarian algorithm.
\STATE Compute the sign matrix $\mat{S}_K$ using  $\mat{P}_K$ and $\mat{B}_K$.
\end{algorithmic}
\end{algorithm}
\begin{figure}[h!]
\begin{center}
\includegraphics[width=0.59\linewidth]{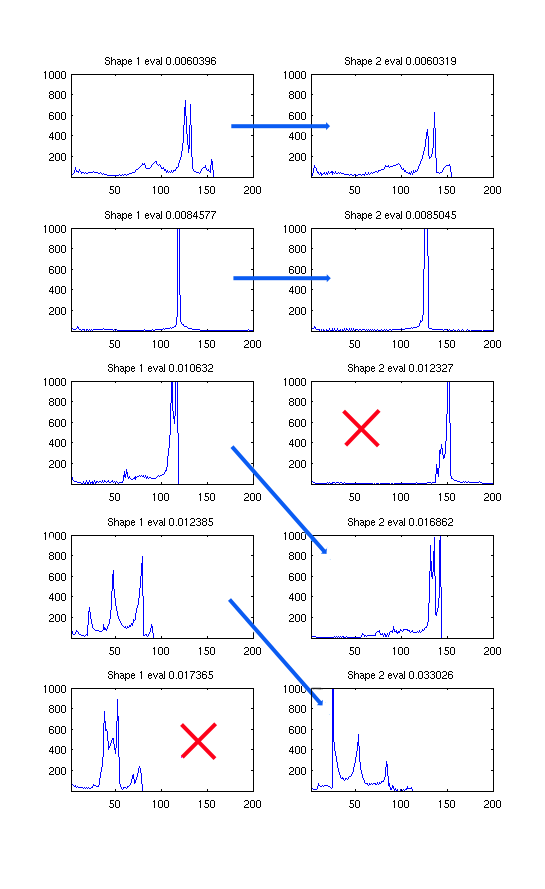} 
\caption{An illustration of applicability of eigenvector histogram as eigensignature to detect sign flip and 
eigenvector ordering change. The blue line shows matched eigenvector pairs and the red-cross depicts discarded eigenvectors.}
\label{Sharma_fig:histogram_matching}
\end{center}
\end{figure}

Figure~\ref{Sharma_fig:histogram_matching} illustrates the utility of the histogram of eigenvectors as
eigensignatures for solving the problem of sign flip and change in eigenvector ordering by computing 
histogram matching\index{histogram matching}. 
It is interesting to observe that a threshold on the histogram matching score (\ref{Sharma_eq:histo_matching_score}) 
allows us to discard the eigenvectors with low similarity cost.  Hence, starting with large $K$ 
obtained using (\ref{Sharma_eq:choosing_K}), we can limit the number of eigenvectors to just a few, 
which will be suitable for EM based point registration algorithm proposed in the next section. 
\subsection{An EM Algorithm for Shape Matching}
\label{Sharma_subsection:PointRegistration}
%In this sub section we will introduce the solution to point registration problem using 
%EM algorithm. 
%
%\textbf{NO! It starts with an alignment based on the estimation of the
%  permutation and sign matrices, used to initialize the EM
%  algorithm. Make the connection to what is said at the end of section
%  1.5.1 (the need of an alternating method).} 

As explained in section~\ref{Sharma_subsection:MaximaSubgraphIsomorphism}, the maximum subgraph matching problem 
reduces to a point registration\index{point registration} problem in $K$ dimensional metric space spanned by the eigenvectors of graph Laplacian 
where two shapes are represented as point clouds. 
The initial alignment of Laplacian embeddings can be obtained by matching the histogram of eigenvectors as
described in the previous section.
In this section we propose an EM algorithm\index{EM algorithm} for 3D shape matching that computes a probabilistic 
vertex-to-vertex assignment between two shapes. The proposed method alternates between the step to estimate 
an orthogonal transformation matrix associated with the alignment of the two shape embeddings and the step 
to compute a point-to-point probabilistic assignment variable.
  
The method is based on a parametric probabilistic model, namely maximum likelihood with missing data. Let
us consider the Laplacian embedding of two shapes, i.e., (\ref{Sharma_eq:X-coordinates})
: $\hat{\mat{X}}=\{\hat{\vec{x}}_i\}_{i=1}^{n},
\hat{\mat{X}}'=\{\hat{\vec{x}}'_j\}_{j=1}^{m}$, with
$\hat{\mat{X}},\hat{\mat{X}}'\subset\mathbbm{R}^K$. Without loss of generality, we
assume that the points in the first set, $\hat{\mat{X}}$ are cluster centers
of a Gaussian mixture model\index{Gaussian mixture model} (GMM) with $n$ clusters and an additional
uniform component that accounts for outliers and unmatched data. The
matching $\hat{\mat{X}} \leftrightarrow \hat{\mat{X}}'$ will consist in fitting the
Gaussian mixture to the set $\hat{\mat{X}}'$.

Let this Gaussian mixture undergo a $K\times K$
transformation $\mat{R}$ (for simplicity, we omit the index $K$) with $\mat{R}\tp\mat{R}=\mat{I}_{K}, \det(\mat{R})=\pm
1$, more precisely $\mat{R}\in \mathcal{O}_K$, the group of orthogonal
matrices acting on $\mathbbm{R}^K$. Hence, each cluster in the mixture
is parametrized by a prior $p_i$, a cluster mean
$\vec{\mu}_i=\mat{R}\hat{\vec{x}}_i$, and a covariance matrix
$\mat{\Sigma}_i$. It will be assumed that all the clusters in the
mixture have the same priors, $\{p_i=\pi_\text{in}\}_{i=1}^{n}$, and the same
isotropic covariance matrix,
$\{\mat{\Sigma}_i=\sigma\mat{I}_{K}\}_{i=1}^{n}$. This
parametrization leads to the following \textit{observed-data
  log-likelihood} (with $\pi_{\text{out}}=1-n\pi_{\text{in}}$ and $\mathcal{U}$ is
the uniform distribution):
\begin{equation}
\label{Sharma_eq:logL}
\log P(\hat{\mat{X}}') = \sum_{j=1}^{m} \log \left( \sum_{i=1}^n \pi_\text{in}
  \mathcal{N}(\hat{\vec{x}}'_j|\vec{\mu}_i,\sigma) + \pi_\text{out}
  \mathcal{U}\right)
\end{equation}
It is well known that the direct maximization of (\ref{Sharma_eq:logL}) is
not tractable and it is more practical to maximize the
\textit{expected complete-data log-likelihood} using the EM
algorithm, where ``complete-data'' refers to both the observed data
(the points $\hat{\mat{X}}'$) and the missing data (the data-to-cluster
assignments). In our case, the above expectation writes (see
\cite{HFYDZ11} for details):
\begin{equation}
\label{Sharma_eq:E}
\mathcal{E}(\mat{R},\sigma)=
-\frac{1}{2}\sum_{j=1}^{m}\sum_{i=1}^{n}\alpha_{ji}(\|\hat{\vec{x}}'_j-\mat{R}\hat{\vec{x}}_i\|^{2}+
k \log \sigma ),
\end{equation}
where $\alpha_{ji}$ denotes the posterior probability of an assignment:
$\hat{\vec{x}}'_j\leftrightarrow\hat{\vec{x}}_i$:
\begin{equation}
\label{Sharma_eq:Post}
\alpha_{ji} =
\frac{\exp(-\|\hat{\vec{x}}'_j-\mat{R}\hat{\vec{x}}_i\|^{2}/2\sigma)}{\sum_{q=1}^{n}\exp(-\|\hat{\vec{x}}'_j-\mat{R}\hat{\vec{x}}_q\|^{2}/2\sigma)
  + \emptyset\sigma^{k/2}},
\end{equation}
where $\emptyset$ is a constant term associated with the uniform
distribution $\mathcal{U}$. Notice that one easily obtains the
posterior probability of a data point to remain unmatched,
$\alpha_{jn+1}=1-\sum_{i=1}^{n}\alpha_{ij}$. This leads to the shape matching\index{shape matching}
procedure outlined in Algorithm~\ref{Sharma_algo:EM}.
\begin{algorithm}[h!]
\caption{\textit{EM for shape matching}}
\label{Sharma_algo:EM}
\begin{algorithmic}[1]
\INPUT : Two embedded shapes $\hat{\mat{X}}$ and $\hat{\mat{X}}'$;
\OUTPUT: Dense correspondences  $\hat{\mat{X}}\leftrightarrow\hat{\mat{X}}'$ between
  the two shapes;
\STATE \textit{Initialization:} Set $\mat{R}^{(0)}=
\mat{S}_K\mat{P}_K$ choose a large value for the variance $\sigma^{(0)}$;
\STATE \textit{E-step:} Compute the current posteriors $\alpha_{ij}^{(q)}$ from the current parameters
using (\ref{Sharma_eq:Post});
\STATE \textit{M-step:} Compute the  new transformation $\mat{R}^{(q+1)}$ and the new variance $\sigma^{(q+1)}$ using the current posteriors:
$$\mat{R}^{(q+1)}=\arg\min_{\mat{R}} \sum_{i,j}\alpha_{ij}^{(q)}
\|\vv{x}'_j -\mat{R}\vv{x}_i\|^2$$ 
$$\sigma^{(q+1)} =
\sum_{i,j}\alpha_{ij}^{(q)}\|\hat{\vec{x}}'_j -\mat{R}^{(q+1)}\hat{\vec{x}}_i\|^2 / k\sum_{i,j}\alpha_{ij}^{(q)}$$
\STATE \textit{MAP:} Accept the assignment $\hat{\vec{x}}'_j \leftrightarrow \hat{\vec{x}}_i$
if $\max_{i} \alpha_{ij}^{(q)}>0.5$.
\end{algorithmic}
\end{algorithm}

%Discuss EM and it's advantage in detail. 

        % Section 5

\section{Experiments and Results}
\label{Sharma_section:Results}

%!!! please use unique labels (e.g., include your initials) for all
%your sections, equations, figures, tables, etc.
%
%In this section we will show results. 

We have performed several 3D shape registration experiments to evaluate the proposed method. 
%In the first experiment we show matching results on TOSCA dataset which consists of
In the first experiment, 3D shape registration is performed on 138 high-resolution 
(10K-50K vertices) triangular meshes from the publicly available TOSCA dataset~\cite{BronsteinBronstein2010a}. 
The dataset includes 3 shape classes (human, dog, horse) with simulated transformations. 
Transformations are split into 9 classes (isometry, topology, small and big holes, global 
and local scaling, noise, shot noise, sampling). Each transformation class appears in five 
different strength levels. An estimate of average geodesic distance to ground truth correspondence 
was computed for performance evaluation (see ~\cite{BronsteinBronstein2010a} for details).

We evaluate our method in two settings. In the first setting SM1 we use the commute-time embedding (\ref{Sharma_eq:X-coordinates}) while in the second setting SM2 we use the unit hyper-sphere normalized embedding (\ref{Sharma_eq:normalized_embedding}).
 
Table~\ref{Sharma_table:results1} shows the error estimates for dense shape matching using 
proposed spectral matching method. 
\begin{table*}
%\begin{small}
\begin{center}
\hspace{1mm}
\begin{tabular}{l c c|c c|c c|c c|c c}
& \multicolumn{10}{c}{{\bf Strength}} \\
\hline
{\bf Transform} & \multicolumn{2}{c}{$1$} & \multicolumn{2}{c}{$\leq 2$} & \multicolumn{2}{c}{$\leq 3$} & \multicolumn{2}{c}{$\leq 4$} & \multicolumn{2}{c}{$\leq 5$} \\
\hline
 & SM1 & SM2 & SM1 & SM2 & SM1 & SM2 & SM1 & SM2 & SM1 & SM2\\
\hline
\textit{Isometry} & 0.00 & 0.00 & 0.00 & 0.00 & 0.00 & 0.00 & 0.00 & 0.00 & 0.00 & 0.00\\
%\hline
\textit{Topology} & 6.89 & 5.96 & 7.92 & 6.76 & 7.92 & 7.14 & 8.04 & 7.55 & 8.41 & 8.13 \\
%\hline
\textit{Holes}    & 7.32 & 5.17 & 8.39 & 5.55 & 9.34 & 6.05 & 9.47 & 6.44 & 12.47 & 10.32\\
%\hline
\textit{Micro holes}    & 0.37 & 0.68 & 0.39 & 0.70 & 0.44 & 0.79 & 0.45 & 0.79 & 0.49 & 0.83\\
%\hline
\textit{Scale} & 0.00 & 0.00 & 0.00 & 0.00 & 0.00 & 0.00 & 0.00 & 0.00 & 0.00 & 0.00 \\
%\hline
\textit{Local scale} & 0.00 & 0.00 & 0.00 & 0.00 & 0.00 & 0.00 & 0.00 & 0.00 & 0.00 & 0.00\\
%\hline
\textit{Sampling} & 11.43 & 10.51 & 13.32 & 12.08 & 15.70 & 13.65 & 18.76 & 15.58 & 22.63 & 19.17\\
%\hline
\textit{Noise} & 0.00 & 0.00 & 0.00 & 0.00 & 0.00 & 0.00 & 0.00 & 0.00 & 0.00 & 0.00 \\
%\hline
\textit{Shot noise} & 0.00 & 0.00 & 0.00 & 0.00 & 0.00 & 0.00 & 0.00 & 0.00 & 0.00 & 0.00 \\
\hline
{\bf Average} & 2.88 & 2.48 & 3.34 & 2.79 & 3.71 & 3.07 & 4.08 &3.37 & 4.89 & 4.27 \\
\hline
\end{tabular}
\end{center}
%\end{small}
\caption{3D shape registration error estimates (average geodesic distance to ground truth correspondences) using proposed spectral matching method with commute-time embedding (SM1) and unit hyper-sphere normalized embedding (SM2).}
\label{Sharma_table:results1}
\end{table*}
%
%
%\begin{table*}
%%\begin{small}
%\begin{center}
%\hspace{1mm}
%\begin{tabular}{l c c c c c}
%& \multicolumn{5}{c}{{\bf Strength}} \\
%\hline
%{\bf Transform} & $1$ &$\leq 2$ & $\leq 3$ & $\leq 4$ & $\leq 5$ \\
%\hline
%\textit{Isometry} & 0.00 & 0.00 & 0.00 & 0.00 & 0.00 \\
%%\hline
%\textit{Topology} & 6.89 & 7.92 & 7.92 & 8.04 & 8.41 \\
%%\hline
%\textit{Holes}    & 7.32 & 8.39 & 9.34 & 9.47 & 12.47 \\
%%\hline
%\textit{Micro holes}    & 0.37 & 0.39 & 0.44 & 0.45 & 0.49 \\
%%\hline
%\textit{Scale} & 0.00 & 0.00 & 0.00 & 0.00 & 0.00 \\
%%\hline
%\textit{Local scale} & 0.00 & 0.00 & 0.00 & 0.00 & 0.00 \\
%%\hline
%\textit{Sampling} & 11.43 & 13.32 & 15.70 & 18.76 & 22.63 \\
%%\hline
%\textit{Noise} & 0.00 & 0.00 & 0.00 & 0.00 & 0.00 \\
%%\hline
%\textit{Shot noise} & 0.00 & 0.00 & 0.00 & 0.00 & 0.00 \\
%\hline
%{\bf Average} & 2.88 & 3.34 & 3.71 & 4.08 & 4.89 \\
%\hline
%\end{tabular}
%\end{center}
%%\end{small}
%\caption{3D shape registration error estimates (average geodesic distance to ground truth correspondences) using proposed spectral matching method (SM1)}
%\label{Sharma_table:results1}
%\end{table*}
%
\begin{table*}
%\begin{small}
\begin{center}
\hspace{1mm}
\begin{tabular}{l c c c c c}
& \multicolumn{5}{c}{{\bf Strength}} \\
\hline
{\bf Method} & $1$ &$\leq 2$ & $\leq 3$ & $\leq 4$ & $\leq 5$ \\
\hline
\textit{LB1} & 10.61 & 15.48 & 19.01 & 23.22 & 23.88 \\
\hline
\textit{LB2} & 15.51 & 18.21 & 22.99 & 25.26 & 28.69 \\
\hline
\textit{GMDS} & 39.92 & 36.77 & 35.24 & 37.40 & 39.10 \\
\hline
\textit{SM1} & 2.88 & 3.34 & 3.71 & 4.08 & 4.89 \\
\hline
\textit{SM2} & 2.48 & 2.79 & 3.07 & 3.37 & 4.27 \\
\hline
\end{tabular}
\end{center}
%\end{small}
\caption{Average shape registration error estimates over all transforms 
(average geodesic distance to ground truth correspondences) computed using proposed methods (SM1 and SM2), 
GMDS~\cite{DubrovinaKimmel2010} and LB1, LB2 ~\cite{BronsteinBronstein2006}.} 
\label{Sharma_table:results2}
\end{table*}
\begin{table*}
%\begin{small}
\begin{center}
\hspace{1mm}
\begin{tabular}{l c c c}
& \multicolumn{3}{c}{{\bf Strength}} \\
\hline
{\bf Transform} & $1$ & $\leq 3$ & $\leq 5$ \\
\hline
\textit{Isometry} & SM1,SM2 & SM1,SM2 & SM1,SM2 \\
%\hline
\textit{Topology} & SM2 & SM2 & SM2 \\
%\hline
\textit{Holes}    & SM2 & SM2 & SM2 \\
%\hline
\textit{Micro holes}  & SM1 & SM1 & SM1 \\
%\hline
\textit{Scale} &  SM1,SM2 & SM1,SM2 & SM1,SM2 \\
%\hline
\textit{Local scale} &  SM1,SM2 & SM1,SM2 & SM1,SM2 \\
%\hline
\textit{Sampling} & LB1 & SM2 & LB2 \\
%\hline 
\textit{Noise} &  SM1,SM2 & SM1,SM2 & SM1,SM2 \\
%\hline
\textit{Shot noise} &  SM1,SM2 & SM1,SM2 & SM1,SM2 \\
\hline
{\bf Average} &  SM1,SM2 & SM1,SM2 & SM1,SM2 \\
\hline
\end{tabular}
\end{center}
%\end{small}
\caption{3D shape registration performance comparison:  The proposed methods (SM1 and SM2) performed best by providing minimum average shape registration error over all the transformation classes with different strength as compare to GMDS~\cite{DubrovinaKimmel2010} and LB1, LB2 ~\cite{BronsteinBronstein2006} methods. }
\label{Sharma_table:results3}
\end{table*}
In the case of some transforms, the proposed method yields zero error because the two meshes had identical 
triangulations. Figure~\ref{Sharma_fig:shrec_results} shows some matching results. The colors emphasize the
correct matching of body parts while we show only $5 \%$ of matches for better visualization. 
In Figure~\ref{Sharma_fig:shrec_results}(e) the two shapes have large difference in the sampling rate.
In this case the matching near the shoulders is not fully correct since we used the commute-time embedding.
 
\begin{figure}[h!]
\begin{center}
\begin{tabular}{ccc}
\includegraphics[width=0.33\linewidth]{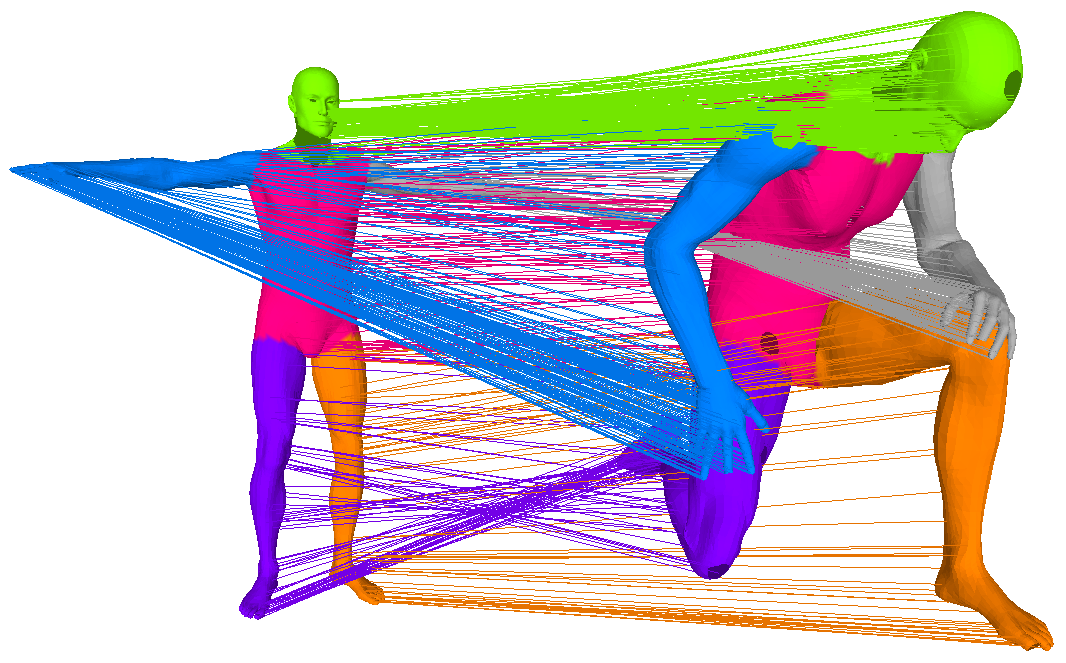} &
\includegraphics[width=0.33\linewidth]{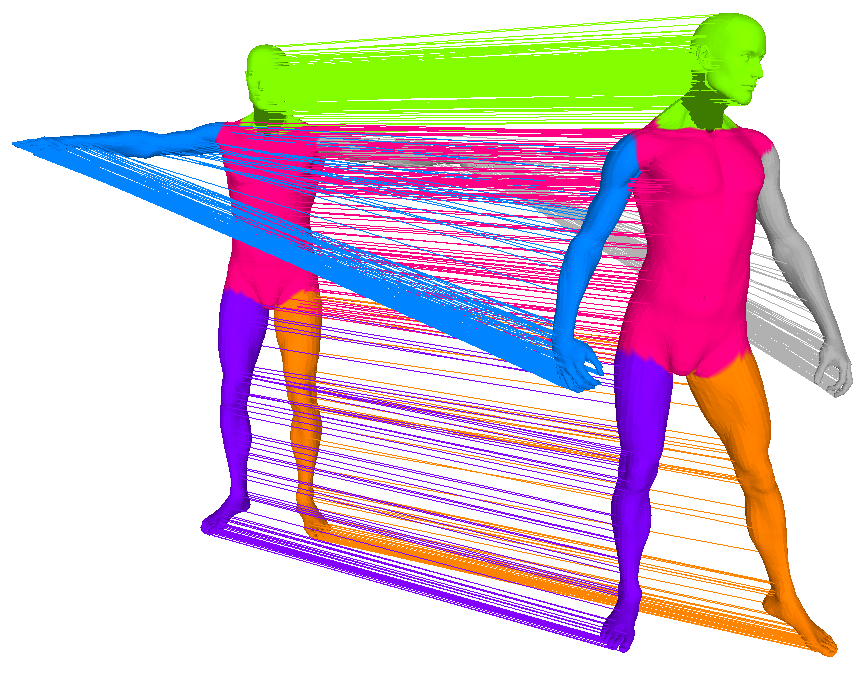} &
\includegraphics[width=0.33\linewidth]{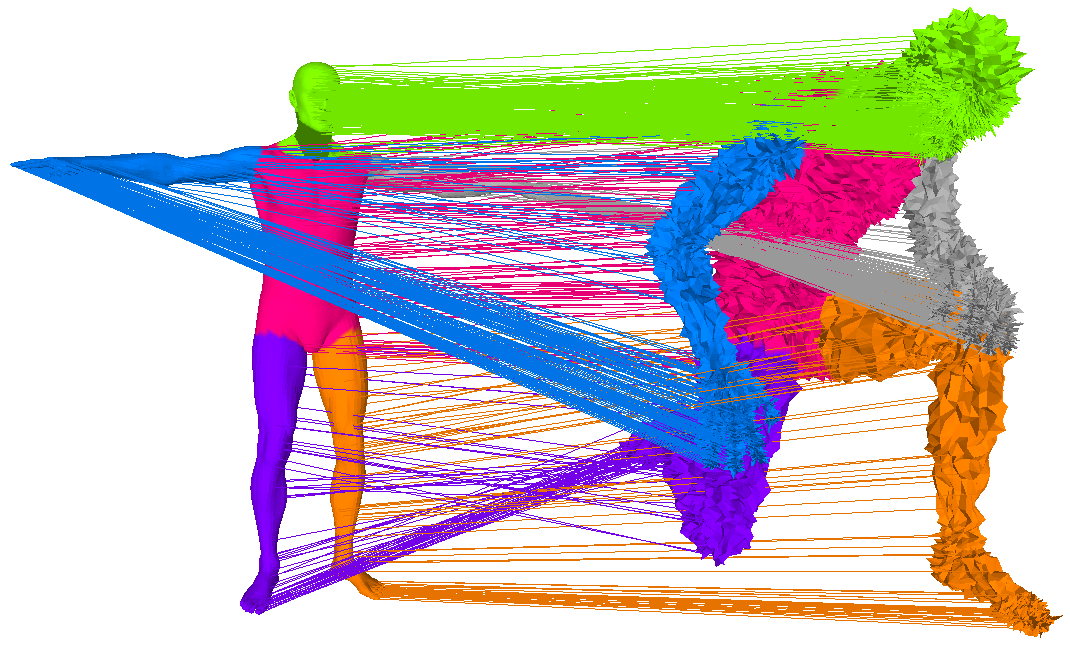} \\
(a) Holes & (b) Isometry & (c) Noise
\end{tabular}
\begin{tabular}{cc}
\includegraphics[width=0.33\linewidth]{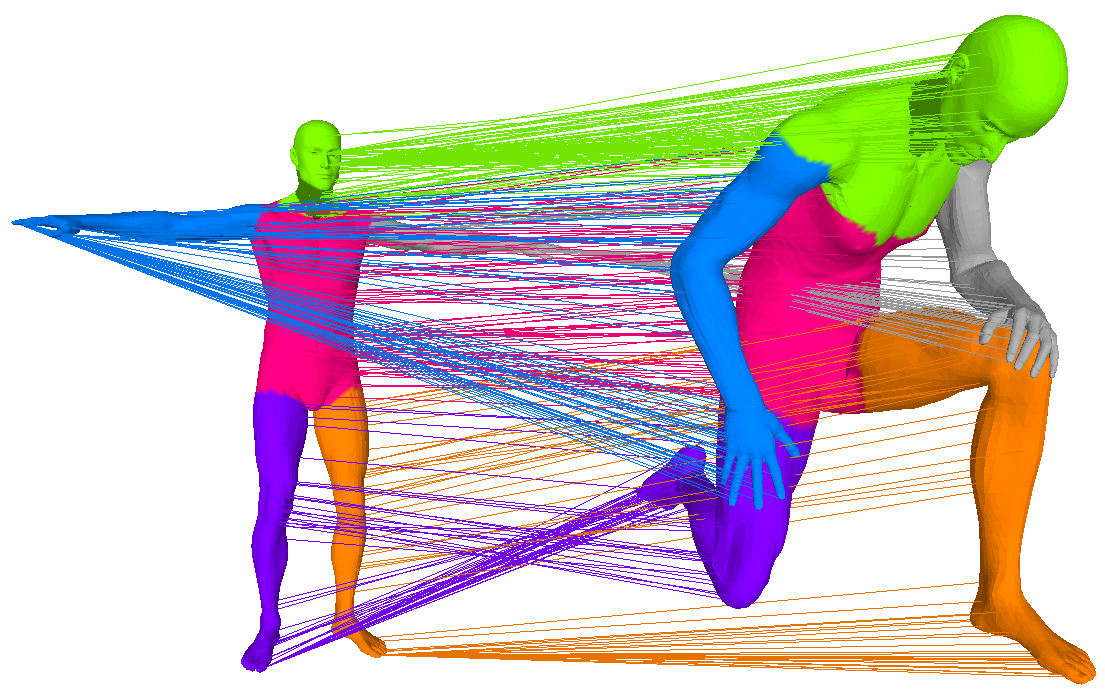} &
\includegraphics[width=0.33\linewidth]{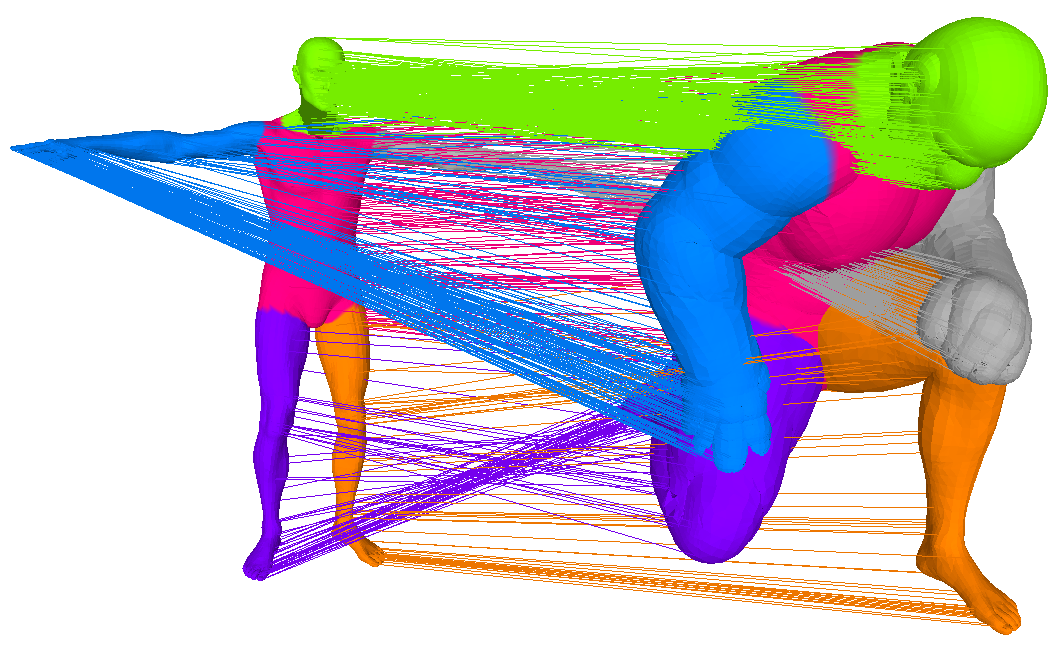} \\
(e) Sampling & (f) Local scale
\end{tabular}
\caption{3D shape registration in the presence of different transforms.}
\label{Sharma_fig:shrec_results}
\end{center}
\end{figure}
Table~\ref{Sharma_table:results2} summarizes the comparison of proposed spectral matching method (SM1 and SM2) with 
generalized multidimensional scaling (GMDS) based matching algorithm introduced in~\cite{DubrovinaKimmel2010} 
and the Laplace-Beltrami matching algorithm proposed in~\cite{BronsteinBronstein2006} with two settings 
LB1 (uses graph Laplacian) and LB2 (uses cotangent weights).
GMDS computes correspondence between two shapes by trying to embed one shape into another with minimum distortion.  
LB1 and LB2 algorithms combines the surface descriptors based on the eigendecomposition of the Laplace-Beltrami 
operator and the geodesic distances measured on the shapes when calculating the correspondence quality. 
The above results in a quadratic optimization problem formulation for correspondence detection, and its minimizer 
is the best possible correspondence.
The proposed method clearly outperform the other two methods with minimum average error estimate computed over all the 
transformations in the dataset.  

In table~\ref{Sharma_table:results3}, we show a detailed comparison of proposed method with other methods. 
%Here we also show performance of variation of our method (SM2) using hyper-sphere normalization of Laplacian
%embedding (see \ref{Sharma_subsection:hypersphere-normalization}) as well
%as the variant of Laplace-Beltrami matching algorithm (LB2) that uses cotangent weights while
%computing Laplacian. 
For a detailed quantitative comparison refer to ~\cite{BronsteinBronstein2010a}. 
The proposed method inherently uses diffusion geometry as opposed to geodesic metric used by other 
two methods and hence outperform them. 

In the second experiment we perform shape registration on two different shapes with similar topology. 
In Figure~\ref{Sharma_fig:inter_shape_results}, results of shape registration on different shapes is presented. 
Figure~\ref{Sharma_fig:inter_shape_results}(a),(c) shows the initialization step of EM algorithm while 
Figure~\ref{Sharma_fig:inter_shape_results}(b),(d) shows the dense matching obtained after EM convergence. 
\begin{figure}[h!]
\begin{center}
\begin{tabular}{cc}
\includegraphics[width=0.49\linewidth]{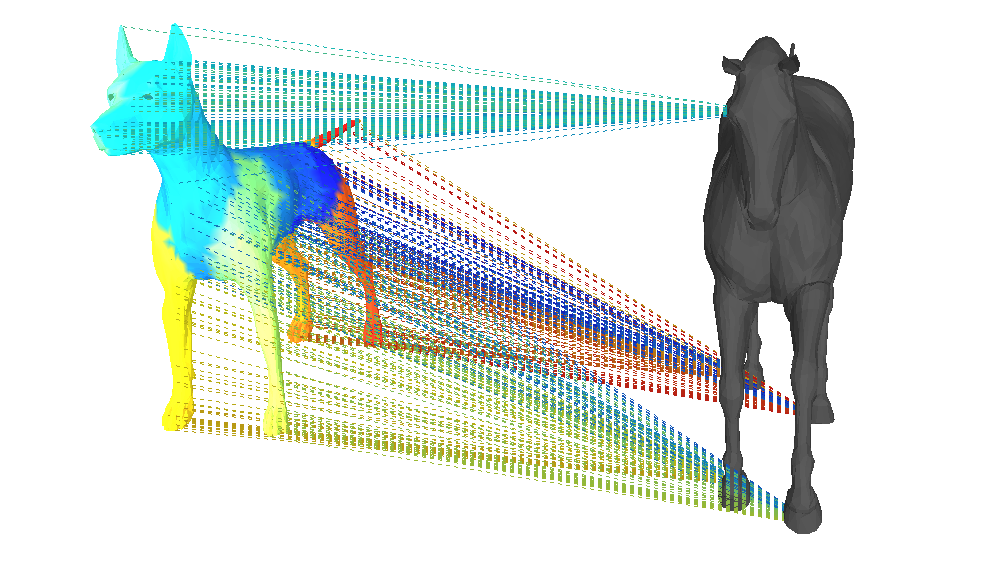} &
\includegraphics[width=0.49\linewidth]{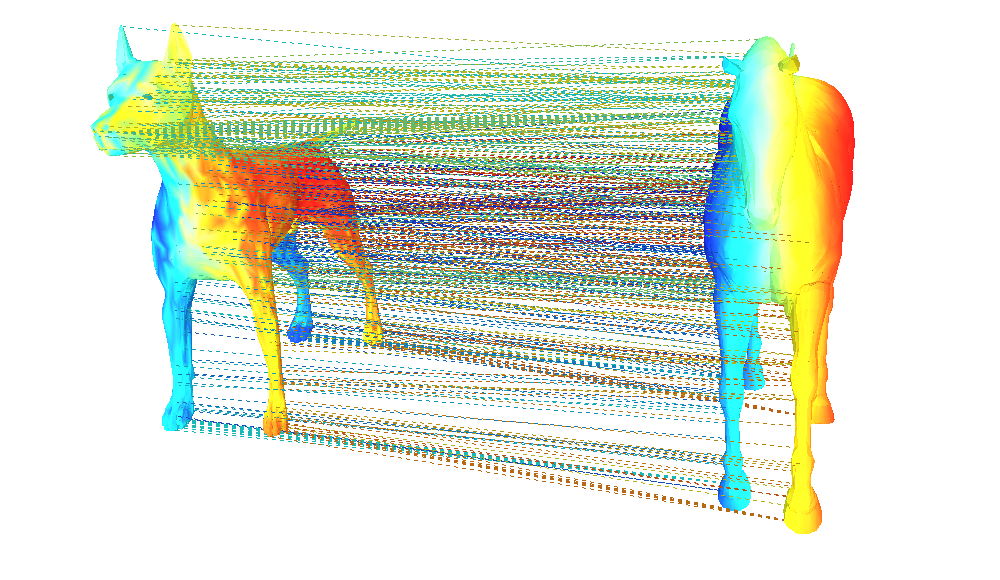} \\
(a) EM Initialization Step & (b) EM Final Step \\
\includegraphics[width=0.49\linewidth]{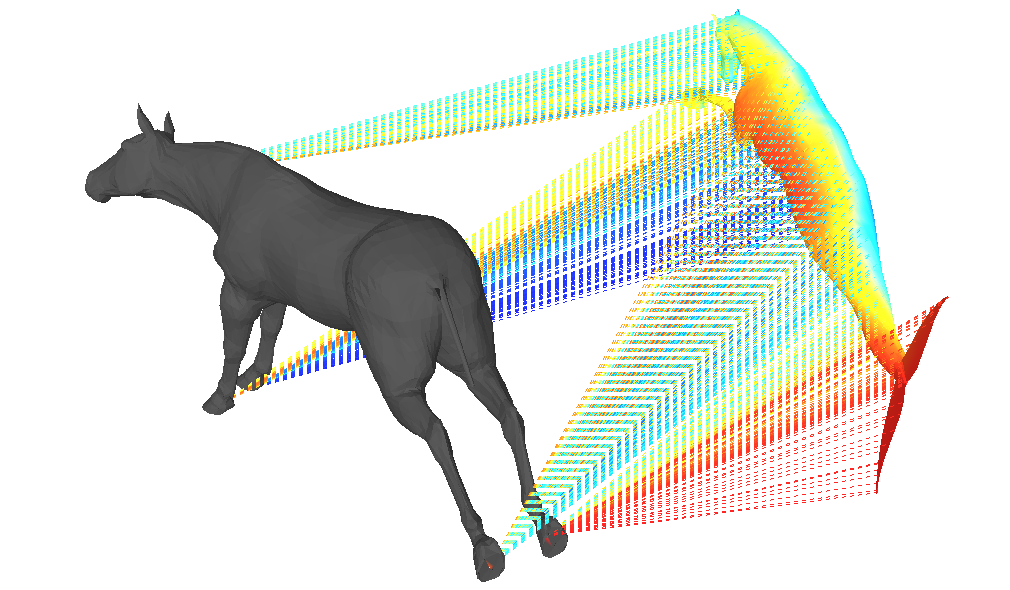} &
\includegraphics[width=0.49\linewidth]{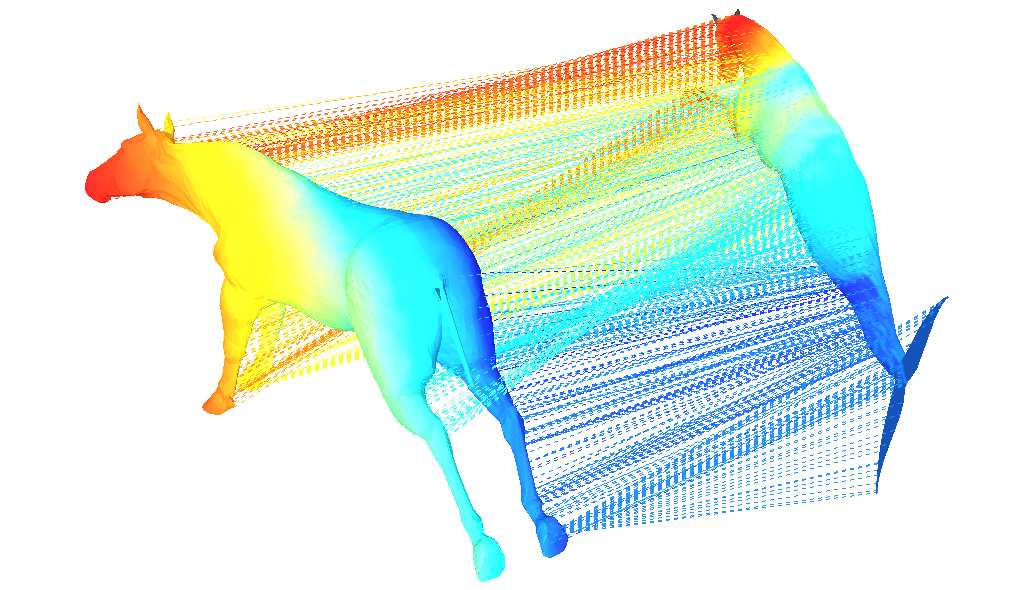} \\
(c) EM Initialization Step & (d) EM Final Step
\end{tabular}
\caption{3D shape registration performed on different shapes with similar topology.}
\label{Sharma_fig:inter_shape_results}
\end{center}
\end{figure}

Finally, we show shape matching results on two different human meshes captured with 
multi-camera system at MIT~\cite{Vlasic-Baran-SIGGRAPH-2008} and University of Surrey~\cite{Starck-Hilton-CGA-2007} 
in Figure~\ref{Sharma_fig:real_shape_results}
\begin{figure}[h!]
\begin{center}
\begin{tabular}{cc}
\includegraphics[width=0.49\linewidth]{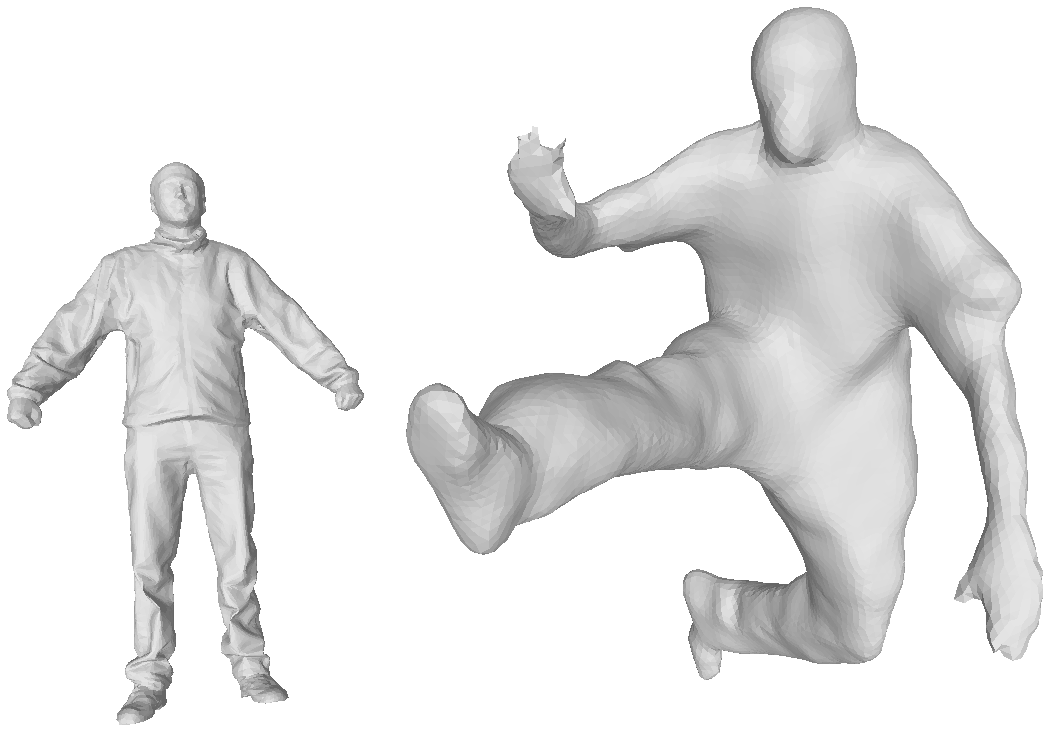} &
\includegraphics[width=0.49\linewidth]{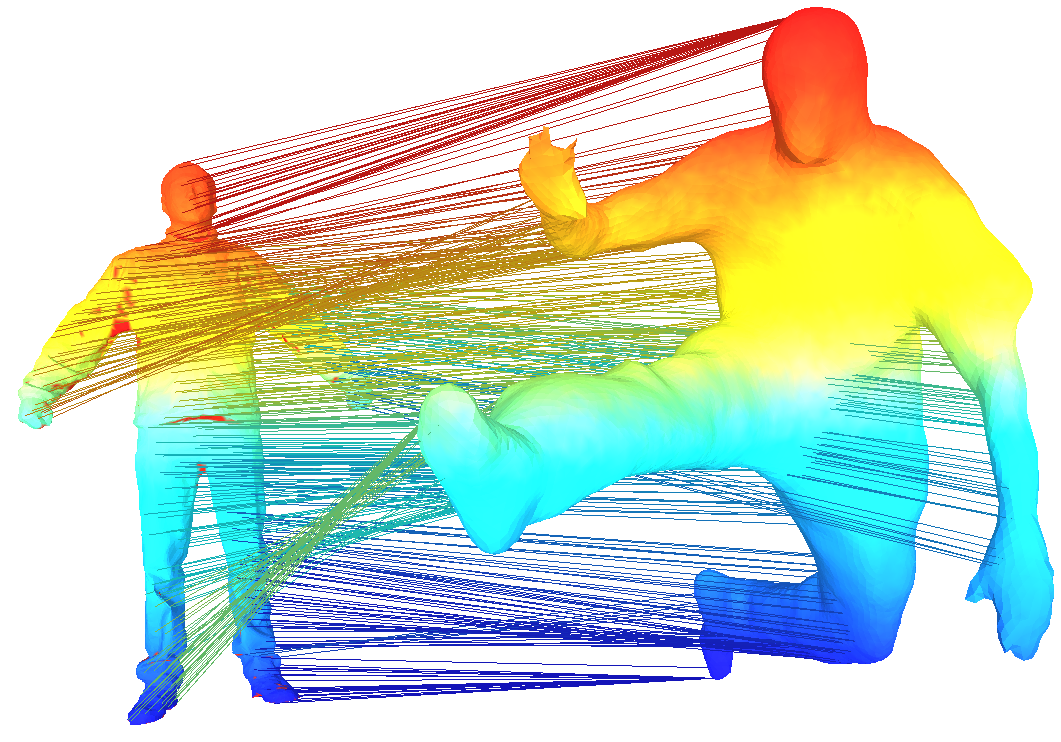} \\
(a) Original Meshes & (b) Dense Matching \\
\end{tabular}
\caption{3D shape registration performed on two real meshes captured from different sequence.}
\label{Sharma_fig:real_shape_results}
\end{center}
\end{figure}

        % Section 6
\section{Discussion}
\label{Sharma_section:discussion}

%!!! please use unique labels (e.g., include your initials) for all
%your sections, equations, figures, tables, etc.
%
%In this section we will discuss the limitation of proposed method. 

This chapter describes a 3D shape registration approach
that computes dense correspondences between two articulated objects. 
We address the problem using spectral matching and unsupervised point registration method. 
We formally introduce graph isomorphism using the Laplacian
matrix, and we provide an analysis of the matching problem
when the number of nodes in the graph is very large, \textit{i.e.} of the
order of $O(10^4)$. We show that there is a simple equivalence between
graph isomorphism and point registration under the
group of orthogonal transformations, when the dimension
of the embedding space is much smaller than the cardinality
of the point-sets.

The eigenvalues of a large sparse Laplacian cannot be reliably
ordered. We propose an elegant alternative to eigenvalue
ordering, using eigenvector histograms and alignment
based on comparing these histograms. The point registration
that results from eigenvector alignment yields an
excellent initialization for the EM algorithm, subsequently
used only to refine the registration. 

However, the method is susceptible to large topology changes that
might occur in the multi-camera shape acquisition setup due to self-occlusion 
(originated from complex kinematics poses) and shadow effects. 
This is because Laplacian embedding is a global representation and
any major topology change will lead to large changes in embeddings causing
failure of this method. Recently, a new shape registration method proposed in~\cite{Sharma-Horadu-CVPR-2011} 
provides robustness to the large topological changes using the heat kernel framework. 
        % Section 7
%\input{Sharma_Acknowledgement} % Acknowledgements (if any)
\appendix
%\chapter{}
\section{Permutation and Doubly-stochastic Matrices}
\label{Sharma_appendix:permutation}
A matrix $\mat{P}$ is called a \textit{permutation} matrix if exactly one
entry in each row and column is equal to $1$, and all other entries
are $0$. Left multiplication of a matrix $\mat{A}$ by a permutation
matrix $\mat{P}$ permutes the \textit{rows} of $\mat{A}$, while right
multiplication permutes the \textit{columns} of $\mat{A}$.

Permutation matrices have the following properties: $\det(\mat{P}) = \pm
1$, $\mat{P}\tp=\mat{P}\inverse$, the identity is a permutation matrix,
and the product of two permutation
matrices is a permutation matrix. Hence the set of permutation
matrices $\mat{P}\in\mathcal{P}_n$ constitute a subgroup of the
subgroup
of orthogonal matrices, denoted by $\mathcal{O}_n$, and
$\mathcal{P}_n$ has finite cardinality $n!$.

A non-negative matrix $\mat{A}$ is a matrix such that all its entries
are non-negative. A non-negative matrix
with the property that all its row sums are $+1$ is said to be a
\textit{(row) stochastic matrix}. A \textit{column stochastic matrix}
is the transpose of a row stochastic matrix. A stochastic matrix
$\mat{A}$ with the property that $\mat{A}\tp$ is also stochastic is said
to be \textit{doubly stochastic}: all row and column sums are $+1$ and
$a_{ij} \geq 0$.
The set of stochastic matrices is a compact convex set with the simple
and important property that $\mat{A}$ is stochastic if and only if
$\mat{A}\mathbbm{1}=\mathbbm{1}$ where $\mathbbm{1}$ is the vector with all components
equal to $+1$.

Permutation matrices are doubly stochastic matrices.\index{doubly stochastic matrices} If we denote by
$\mathcal{D}_n$ the set of doubly stochastic
matrices, it can be proved that $\mathcal{P}_n = \mathcal{O}_n \cap
\mathcal{D}_n$ \cite{ZavlanosPappas2008}.
The permutation matrices are the fundamental and prototypical
doubly stochastic matrices, for Birkhoff's theorem states that any
doubly stochastic matrix is a linear convex combination of finitely many
permutation matrices \cite{HornJohnson94}:
\begin{theorem}
\label{Sharma_theorem:Birkhoff}
(Birkhoff) A matrix $ \mat{A}$ is a doubly stochastic matrix if and only
if for some $N<\infty$ there are permutation matrices
$\mat{P}_1,\hdots,\mat{P}_N$ and positive scalars $s_1,\hdots,s_N$ such
that $s_1+\hdots + s_N=1$ and $\mat{A}=s_1 \mat{P}_1 +\hdots + s_N \mat{P}_N$.
\end{theorem}

A complete proof of this theorem is to be found in \cite{HornJohnson94}[pages
526--528]. The proof relies on the fact that
$\mathcal{D}_n$ is a compact convex set and every point in such a set
is a convex combination of the extreme points of the set. First it is
proved that every permutation matrix is an extreme point of
$\mathcal{D}_n$ and second it is shown that a given matrix
is an extreme point of $\mathcal{D}_n$ if an only if it is a
permutation matrix.

%
%\section{Basic properties of matrix traces}
%\label{Sharma_appendix:traces}
%
%The trace of an $n\times n$ matrix $\adjmat$ is defined by:
%\begin{equation}
%\trace(\adjmat) =  \sum_{i=1}^{n}a_{ij} = \trace(\adjmat\tp)
%\label{Sharma_eq:trace}
%\end{equation}
%From the definition of matrix multiplication we get:
%\begin{equation}
%\trace(\adjmat\adjmat_B) =   \sum_{i=1}^{n}\sum_{j=1}^{n} a_{ij}b_{ji} = \trace(\adjmat_B\adjmat)
%%\label{Sharma_eq:trace-of-product}
%\end{equation}
%The trace is a linear map:
%\begin{eqnarray}
%\label{Sharma_eq:trace-add}
%\trace(\adjmat+\adjmat_B) &=& \trace(\adjmat) + \trace(\adjmat_B) \\
% \trace(c\adjmat) &=& c  \trace(\adjmat)
% \end{eqnarray}
%The trace is similarity invariant. This means that for any invertible matrix:
%\begin{equation}
%\trace(\mat{R}\adjmat\mat{R}\inverse) = \trace(\adjmat)
%\end{equation}
%and for any orthogonal matrix:
%\begin{equation}
%\trace(\mat{Q}\adjmat\mat{Q}\tp) = \trace(\adjmat)
%\end{equation}
%For a matrix $\adjmat$  with $n$ \textit{distinct} eigenvalues:
%\begin{equation}
%\trace(\adjmat) = \sum_{i=1}^{n} \alpha_i
%\end{equation}
%More generally, if the characteristic polynomial of a matrix is:
%\[
%\det(\adjmat-\alpha\mat{I}) = (\alpha-\alpha_1)^{m_1} \hdots (\alpha-\alpha_k)^{m_k}
%\]
%where $m_i$ denotes the algebraic multiplicity of eigenvalue $\alpha_i$, then:
%\[
%\trace(\adjmat) = m_1 \alpha_1 + \hdots + m_k \alpha_k
%\]
\section{The Frobenius Norm}
\label{Sharma_appendix:Frobenius}
The Frobenius (or Euclidean) norm of a matrix $\mat{A}_{n \times n}$ is an \textit{entry-wise} norm 
that treats the matrix as a vector of size $1 \times nn$. The standard norm properties hold:
$\frobenius{\mat{A}}>0 \Leftrightarrow \mat{A}\neq 0$, $\frobenius{\mat{A}}=0 \Leftrightarrow \mat{A}= 0$, 
$\frobenius{c\mat{A}}= c\frobenius{\mat{A}}$, and $\frobenius{\mat{A}+\mat{B}} \leq \frobenius{\mat{A}} + \frobenius{\mat{B}}$. 
Additionally, the Frobenius norm is \textit{sub-multiplicative}:
\begin{equation}
\frobenius{\mat{A}\mat{B}} \leq \frobenius{\mat{A}} \frobenius{\mat{B}}
\label{Sharma_eq:sub-multiplicative}
\end{equation}
as well as \textit{unitarily-invariant}. This means that for any two orthogonal matrices $\mat{U}$ and $\mat{V}$:
\begin{equation}
\frobenius{\mat{U} \mat{A} \mat{V} } = \frobenius{\mat{A}}.
\label{Sharma_eq:unit-invariant}
\end{equation}
It immediately follows the following equalities:
\begin{equation}
\frobenius{\mat{U} \mat{A} \mat{U}\tp} =  \frobenius{\mat{U} \mat{A}} = \frobenius{\mat{A} \mat{U}}= \frobenius{\mat{A}}.
\end{equation}

\section{Spectral Properties of the Normalized Laplacian}
\label{Sharma_appendix:Properties-normalized-laplacian}
\paragraph*{The normalized Laplacian} 
%Its spectral decomposition is $\normlapmat=\tilde{\mat{U}}\mat{\Gamma}\tilde{\mat{U}}\tp$. 
Let $\tilde{\vec{u}}_k$ and $\gamma_k$ denote the eigenvectors and
eigenvalues of $\normlapmat$; The
spectral decomposition is
%\begin{equation}
$\normlapmat=\tilde{\mat{U}}\mat{\Gamma}\tilde{\mat{U}}\tp$
%\label{eq:normalized-spectral-dec}
%\end{equation}
 with
$\tilde{\mat{U}}\tilde{\mat{U}}\tp=\mat{I}$. The smallest eigenvalue and associated
eigenvector are $\gamma_1=0$ and $\tilde{\vec{u}}_1=\mat{D}^{1/2}\mathbbm{1}$.

We obtain the following equivalent relations:
\begin{eqnarray}
\label{Sharma_eq:sum-of-entries-2}
&\sum_{i=1}^n d_i^{1/2}\tilde{u}_{ik} = 0,&  2 \leq k \leq n \\
\label{Sharma_eq:limits-of-entries-2}
& d_i^{1/2} | \tilde{u}_{ik}| < 1, & 1\leq i \leq n, 2 \leq k \leq n.
\end{eqnarray}

Using (\ref{Sharma_eq:comb-from-rest}) we obtain a useful expression for the
combinatorial Laplacian in terms of the spectral decomposition of the
normalized Laplacian. Notice, however, that the expression below is
NOT a spectral decomposition of the combinatorial Laplacian:
\begin{equation}
\label{Sharma_eq:comb-from-normal-spectrum}
\lapmat = (\degmat^{1/2}\tilde{\mat{U}}\mat{\Gamma}^{1/2})(\degmat^{1/2}\tilde{\mat{U}}\mat{\Gamma}^{1/2})\tp.
\end{equation}

For a connected graph $\gamma_1$ has multiplicity 1:
$0=\gamma_1 < \gamma_2 \leq \hdots \leq \gamma_n$. As in the case of
the combinatorial Laplacian, there is an upper bound on the
eigenvalues (see \cite{Chung97} for a proof):
\begin{proposition}
\label{Sharma_prop:max-mu}
For all $k\leq n$, we have $\mu_k \leq 2 $.
\end{proposition}
We obtain the
following spectral decomposition for the normalized Laplacian
%and its pseudo-inverse
:
\begin{equation}
\normlapmat = \sum_{k=2}^{n} \gamma_k \tilde{\vec{u}}_k \tilde{\vec{u}}_k\tp.
\label{Sharma_eq:normalized-spectral-dec}
\end{equation}
%\begin{equation}
%\normlapmat^{\dag} = \sum_{k=2}^{n} \frac{1}{\gamma_k} \tilde{\vec{u}}_k \tilde{\vec{u}}_k\tp
%\end{equation}
The spread of the graph along the $k$-th normalized Laplacian
eigenvector is given by $\forall (k,i),  2 \leq k \leq
n, 1\leq i \leq n$:
\begin{eqnarray}
\label{Sharma_eq:sum-of-entries-2}
& \overline{\tilde{u}}_k = & \frac{1}{n} \sum_{i=1}^n \tilde{u}_{ik}\\
&\sigma_{u_k}  = & \frac{1}{n} -  \overline{\tilde{u}}_k^2.
\end{eqnarray}
Therefore, the projection of the graph onto an eigenvector $\tilde{\vec{u}}_k$
is not centered.
By combining  (\ref{Sharma_eq:comb-from-rest}) and
(\ref{Sharma_eq:normalized-spectral-dec}) we obtain an alternative
representation of the combinatorial Laplacian in terms of the the
spectrum of the normalized Laplacian, namely:
\begin{equation}
\lapmat = \sum_{k=2}^{n} \gamma_k (\mat{D}^{1/2}\tilde{\vec{u}}_k)( \mat{D}^{1/2} \tilde{\vec{u}}_k)\tp.
\end{equation}
Hence, an alternative is to project the graph onto the vectors
$\vec{t}_k=\mat{D}^{1/2}\tilde{\vec{u}}_k$. From $\tilde{\vec{u}}_{k\geq 2}\tp \tilde{\vec{u}}_1=0$
we get that $\vec{t}_{k\geq 2}\tp \mathbbm{1}=0$. Therefore, the spread
of the graph's projection onto $\vec{t}_k$ has the following mean and
variance, $\forall (k,i),  2 \leq k \leq
n, 1\leq i \leq n$:
\begin{eqnarray}
\label{Sharma_eq:sum-of-entries-2}
& \overline{t}_k = \sum_{i=1}^n d_i^{1/2}\tilde{u}_{ik} = & 0 \\
& \sigma_{t_k} = \frac{1}{n} \sum_{i=1}^{n} d_i \tilde{u}_{ik}^2.
\end{eqnarray}
\paragraph*{The random-walk Laplacian.} This operator is not
symmetric, however its spectral properties can be easily derived from
those of the normalized Laplacian using
(\ref{Sharma_eq:random-from-rest}). Notice that this can be used to transform
a non-symmetric Laplacian into a symmetric one, as proposed in
\cite{SunOvsjanikovGuibas2009} and in \cite{LuoSafaWang2009}.

        % Appendix (if any)

%{\small \itemsep=-.2cm \bibliography{Sharma_bibtexfile}} % Bibliography

\begin{thebibliography}{10}
\providecommand{\url}[1]{#1}
\csname url@rmstyle\endcsname
\providecommand{\newblock}{\relax}
\providecommand{\bibinfo}[2]{#2}
\providecommand\BIBentrySTDinterwordspacing{\spaceskip=0pt\relax}
\providecommand\BIBentryALTinterwordstretchfactor{4}
\providecommand\BIBentryALTinterwordspacing{\spaceskip=\fontdimen2\font plus
\BIBentryALTinterwordstretchfactor\fontdimen3\font minus
  \fontdimen4\font\relax}
\providecommand\BIBforeignlanguage[2]{{%
\expandafter\ifx\csname l@#1\endcsname\relax
\typeout{** WARNING: IEEEtran.bst: No hyphenation pattern has been}%
\typeout{** loaded for the language `#1'. Using the pattern for}%
\typeout{** the default language instead.}%
\else
\language=\csname l@#1\endcsname
\fi
#2}}

\bibitem{Franco-Boyer-PAMI-2008}
J.-S. {F}ranco and E.~{B}oyer, ``{E}fficient {P}olyhedral {M}odeling from
  {S}ilhouettes,'' \emph{IEEE Transactions on Pattern Analysis and Machine
  Intelligence}, vol.~31, no.~3, p. 414–427, 2009.

\bibitem{Starck-Hilton-CGA-2007}
J.~Starck and A.~Hilton, ``Surface capture for performance based animation,''
  \emph{IEEE Computer Graphics and Applications}, vol.~27, no.~3, pp. 21--31,
  2007.

\bibitem{Slabaugh-Malzbender-WVG-2001}
G.~Slabaugh, B.~Culbertson, T.~Malzbender, and R.~Schafer, ``A survey of
  methods for volumetric scene reconstruction from photographs,'' in
  \emph{International Workshop on Volume Graphics}, 2001, pp. 81--100.

\bibitem{Seitz-Brian-CVPR-2006}
S.~M. Seitz, B.~Curless, J.~Diebel, D.~Scharstein, and R.~Szeliski, ``A
  comparison and evaluation of multi-view stereo reconstruction algorithms,''
  in \emph{IEEE Computer Society Conference on Computer Vision and Pattern
  Recognition}, 2006, pp. 519--528.

\bibitem{Vlasic-Baran-SIGGRAPH-2008}
D.~Vlasic, I.~Baran, W.~Matusik, and J.~Popovic, ``Articulated mesh animation
  from multi-view silhouettes,'' \emph{ACM Transactions on Graphics (Proc.
  SIGGRAPH)}, vol.~27, no.~3, pp. 97:1--97:9, 2008.

\bibitem{ZBH11}
A.~Zaharescu, E.~Boyer, and R.~P. Horaud, ``Topology-adaptive mesh deformation
  for surface evolution, morphing, and multi-view reconstruction,'' \emph{IEEE
  Transactions on Pattern Analysis and Machine Intelligence}, vol.~33, no.~4,
  pp. 823 -- 837, April 2011.

\bibitem{ChenGerard92}
Y.~Chen and G.~Medioni, ``Object modelling by registration of multiple range
  images,'' \emph{Image Vision Computing}, vol.~10, pp. 145--155, April 1992.

\bibitem{BeslMcKay92}
P.~J. Besl and N.~D. McKay, ``A method for registration of 3-d shapes,''
  \emph{IEEE Transactions on Pattern Analysis and Machine Intelligence},
  vol.~14, pp. 239--256, February 1992.

\bibitem{RusinkiewiczLevoy-3DIM-2001}
S.~Rusinkiewicz and M.~Levoy, ``Efficient variants of the {ICP} algorithm,'' in
  \emph{International Conference on 3D Digital Imaging and Modeling}, 2001, pp.
  145--152.

\bibitem{Umeyama88}
S.~Umeyama, ``An eigendecomposition approach to weighted graph matching
  problems,'' \emph{IEEE Transactions on Pattern Analysis and Machine
  Intelligence}, vol.~10, no.~5, pp. 695--703, May 1988.

\bibitem{Wilkinson70}
J.~H. Wilkinson, ``Elementary proof of the {Wielandt-Hoffman} theorem and of
  its generalization,'' Stanford University, Tech. Rep. CS150, January 1970.

\bibitem{BronsteinBronstein2006}
A.~Bronstein, M.~Bronstein, and R.~Kimmel, ``Generalized multidimensional
  scaling: a framework for isometry-invariant partial surface matching,''
  \emph{Proceedings of National Academy of Sciences}, vol. 103, pp. 1168--1172,
  2006.

\bibitem{WangWang2007}
S.~Wang, Y.~Wang, M.~Jin, X.~Gu, D.~Samaras, and P.~Huang, ``Conformal geometry
  and its application on 3d shape matching,'' \emph{IEEE Transactions on
  Pattern Analysis and Machine Intelligence}, vol.~29, no.~7, pp. 1209--1220,
  2007.

\bibitem{JainZhang2007}
V.~Jain, H.~Zhang, and O.~van Kaick, ``Non-rigid spectral correspondence of
  triangle meshes,'' \emph{International Journal of Shape Modeling}, vol.~13,
  pp. 101--124, 2007.

\bibitem{ZengZeng2008}
W.~Zeng, Y.~Zeng, Y.~Wang, X.~Yin, X.~Gu, and D.~Samras, ``3d non-rigid surface
  matching and registration based on holomorphic differentials,'' in
  \emph{European Conference on Computer Vision}, 2008, pp. 1--14.

\bibitem{mateus:cvpr2008}
D.~Mateus, R.~Horaud, D.~Knossow, F.~Cuzzolin, and E.~Boyer, ``Articulated
  shape matching using {L}aplacian eigenfunctions and unsupervised point
  registration,'' in \emph{IEEE Computer Society Conference on Computer Vision
  and Pattern Recognition}, 2008, pp. 1--8.

\bibitem{RuggeriPatane2009}
M.~R. Ruggeri, G.~Patan\'e, M.~Spagnuolo, and D.~Saupe, ``Spectral-driven
  isometry-invariant matching of 3d shapes,'' \emph{International Journal of
  Computer Vision}, vol.~89, pp. 248--265, 2010.

\bibitem{Lipman2009}
Y.~Lipman and T.~Funkhouser, ``Mobius voting for surface correspondence,''
  \emph{ACM Transactions on Graphics ( Proc. SIGGRAPH)}, vol.~28, no.~3, pp.
  72:1--72:12, 2009.

\bibitem{DubrovinaKimmel2010}
A.~Dubrovina and R.~Kimmel, ``Matching shapes by eigendecomposition of the
  {Laplace-Beltrami} operator,'' in \emph{International Symposium on 3D Data
  Processing, Visualization and Transmission}, 2010.

\bibitem{ScottHiggins91}
G.~Scott and C.~L. Higgins, ``{An Algorithm for Associating the Features of Two
  Images},'' \emph{Biological Sciences}, vol. 244, no. 1309, pp. 21--26, 1991.

\bibitem{ShapiroBrady92}
L.~S. Shapiro and J.~M. Brady, ``Feature-based correspondence: an eigenvector
  approach,'' \emph{Image Vision Computing}, vol.~10, pp. 283--288, June 1992.

\bibitem{LuoHancock2001}
B.~Luo and E.~R. Hancock, ``Structural graph matching using the em algorithm
  and singular value decomposition,'' \emph{IEEE Transactions on Pattern
  Analysis and Machine Intelligence}, vol.~23, pp. 1120--1136, October 2001.

\bibitem{WangHankcock2006PR}
H.~F. Wang and E.~R. Hancock, ``Correspondence matching using kernel principal
  components analysis and label consistency constraints,'' \emph{Pattern
  Recognition}, vol.~39, pp. 1012--1025, June 2006.

\bibitem{QiuHancock2007PR}
H.~Qiu and E.~R. Hancock, ``Graph simplification and matching using commute
  times,'' \emph{Pattern Recognition}, vol.~40, pp. 2874--2889, October 2007.

\bibitem{LeordeanuHebert2005}
M.~Leordeanu and M.~Hebert, ``A spectral technique for correspondence problems
  using pairwise constraints,'' in \emph{International Conference on Computer
  Vision}, 2005, pp. 1482--1489.

\bibitem{DuchenneBach2009}
O.~Duchenne, F.~Bach, I.~Kweon, and J.~Ponce, ``A tensor based algorithm for
  high order graph matching,'' in \emph{IEEE Computer Society Conference on
  Computer Vision and Pattern Recognition}, 2009, pp. 1980--1987.

\bibitem{TorresaniKolmogorov2008}
L.~Torresani, V.~Kolmogorov, and C.~Rother, ``Feature correspondence via graph
  matching : Models and global optimazation,'' in \emph{European Conference on
  Computer Vision}, 2008, pp. 596--609.

\bibitem{ZassShahua2008}
R.~Zass and A.~Shashua, ``Probabilistic graph and hypergraph matching,'' in
  \emph{IEEE Computer Society Conference on Computer Vision and Pattern
  Recognition}, 2008, pp. 1--8.

\bibitem{MacielCosteira2003}
J.~Maciel and J.~P. Costeira, ``A global solution to sparse correspondence
  problems,'' \emph{IEEE Transactions on Pattern Analysis and Machine
  Intelligence}, vol.~25, pp. 187--199, 2003.

\bibitem{HuangAdams2008}
Q.~Huang, B.~Adams, M.~Wicke, and L.~J. Guibas, ``Non-rigid registration under
  isometric deformations,'' \emph{Computer Graphics Forum}, vol.~27, no.~5, pp.
  1449--1457, 2008.

\bibitem{ZengWang2010}
Y.~Zeng, C.~Wang, Y.~Wang, X.~Gu, D.~Samras, and N.~Paragios, ``Dense non-rigid
  surface registration using high order graph matching,'' in \emph{IEEE
  Computer Society Conference on Computer Vision and Pattern Recognition},
  2010, pp. 382--389.

\bibitem{SahilliogluYemez2010}
Y.~Sahillioglu and Y.~Yemez, ``3d shape correspondence by isometry-driven
  greedy optimization,'' in \emph{IEEE Computer Society Conference on Computer
  Vision and Pattern Recognition}, 2010, pp. 453--458.

\bibitem{BronsteinBronstein2010a}
A.~M. Bronstein, M.~M. Bronstein, U.~Castellani, A.~Dubrovina, L.~J. Guibas,
  R.~P. Horaud, R.~Kimmel, D.~Knossow, E.~v. Lavante, M.~D., M.~Ovsjanikov, and
  A.~Sharma, ``Shrec 2010: robust correspondence benchmark,'' in
  \emph{Eurographics Workshop on 3D Object Retrieval}, 2010.

\bibitem{Ovsjanikov-Merigot-SGP-2010}
M.~Ovsjanikov, Q.~Merigot, F.~Memoli, and L.~Guibas, ``One point isometric
  matching with the heat kernel,'' \emph{Computer Graphics Forum (Proc. SGP)},
  vol.~29, no.~5, pp. 1555--1564, 2010.

\bibitem{Sharma-Horaud-NORDIA-2010}
A.~Sharma and R.~Horaud, ``Shape matching based on diffusion embedding and on
  mutual isometric consistency,'' in \emph{NORDIA workshop IEEE Computer
  Society Conference on Computer Vision and Pattern Recognition}, 2010.

\bibitem{Sharma-Horadu-CVPR-2011}
A.~Sharma, R.~Horaud, J.~Cech, and E.~Boyer, ``Topologically-robust 3d shape
  matching based on diffusion geometry and seed growing,'' in \emph{IEEE
  Computer Society Conference on Computer Vision and Pattern Recognition},
  2011.

\bibitem{KnossowSMH09}
D.~Knossow, A.~Sharma, D.~Mateus, and R.~Horaud, ``Inexact matching of large
  and sparse graphs using laplacian eigenvectors,'' in \emph{Graph-Based
  Representations in Pattern Recognition}, 2009, pp. 144--153.

\bibitem{HFYDZ11}
R.~P. Horaud, F.~Forbes, M.~Yguel, G.~Dewaele, and J.~Zhang, ``Rigid and
  articulated point registration with expectation conditional maximization,''
  \emph{IEEE Transactions on Pattern Analysis and Machine Intelligence},
  vol.~33, no.~3, pp. 587--602, 2011.

\bibitem{belkin2003laplacian}
M.~Belkin and P.~Niyogi, ``{Laplacian eigenmaps for dimensionality reduction
  and data representation},'' \emph{Neural computation}, vol.~15, no.~6, pp.
  1373--1396, 2003.

\bibitem{Luxburg2007}
U.~von Luxburg, ``A tutorial on spectral clustering,'' \emph{Statistics and
  Computing}, vol.~17, no.~4, pp. 395--416, 2007.

\bibitem{Chung97}
F.~R.~K. Chung, \emph{Spectral Graph Theory}.\hskip 1em plus 0.5em minus
  0.4em\relax American Mathematical Society, 1997.

\bibitem{grady2010discrete}
L.~Grady and J.~R. Polimeni, \emph{{Discrete Calculus: Applied Analysis on
  Graphs for Computational Science}}.\hskip 1em plus 0.5em minus 0.4em\relax
  Springer, 2010.

\bibitem{GodsilRoyle2001}
C.~Godsil and G.~Royle, \emph{Algebraic Graph Theory}.\hskip 1em plus 0.5em
  minus 0.4em\relax Springer, 2001.

\bibitem{HoffmanWielandt53}
A.~J. Hoffman and H.~W. Wielandt, ``The variation of the spectrum of a normal
  matrix,'' \emph{Duke Mathematical Journal}, vol.~20, no.~1, pp. 37--39, 1953.

\bibitem{Wilkinson65}
J.~H. Wilkinson, \emph{The Algebraic Eigenvalue Problem}.\hskip 1em plus 0.5em
  minus 0.4em\relax Oxford: Clarendon Press, 1965.

\bibitem{HornJohnson94}
R.~A. Horn and C.~A. Johnson, \emph{Matrix Analysis}.\hskip 1em plus 0.5em
  minus 0.4em\relax Cambridge: Cambridge University Press, 1994.

\bibitem{burkard2009}
R.~Burkard, \emph{Assignment Problems}.\hskip 1em plus 0.5em minus 0.4em\relax
  Philadelphia: SIAM, Society for Industrial and Applied Mathematics, 2009.

\bibitem{Ham-Lee-ICML-2004}
J.~Ham, D.~D. Lee, S.~Mika, and B.~Sch\"{o}lkopf, ``A kernel view of the
  dimensionality reduction of manifolds,'' in \emph{International Conference on
  Machine Learning}, 2004, pp. 47--54.

\bibitem{QiuHancock2007PAMI}
H.~Qiu and E.~R. Hancock, ``Clustering and embedding using commute times,''
  \emph{IEEE Transactions on Pattern Analysis and Machine Intelligence},
  vol.~29, no.~11, pp. 1873--1890, 2007.

\bibitem{GrinsteadSnell98}
C.~M. Grinstead and L.~J. Snell, \emph{Introduction to Probability}.\hskip 1em
  plus 0.5em minus 0.4em\relax American Mathematical Society, 1998.

\bibitem{Scott1979}
D.~W. Scott, ``On optimal and data-based histograms,'' \emph{Biometrika},
  vol.~66, no.~3, pp. 605--610, 1979.

\bibitem{ZavlanosPappas2008}
M.~M. Zavlanos and G.~J. Pappas, ``A dynamical systems approach to weighted
  graph matching,'' \emph{Automatica}, vol.~44, pp. 2817--2824, 2008.

\bibitem{SunOvsjanikovGuibas2009}
J.~Sun, M.~Ovsjanikov, and L.~Guibas, ``A concise and provably informative
  multi-scale signature based on heat diffusion,'' in \emph{SGP}, 2009.

\bibitem{LuoSafaWang2009}
C.~Luo, I.~Safa, and Y.~Wang, ``Approximating gradients for meshes and point
  clouds via diffusion metric,'' \emph{Computer Graphics Forum (Proc. SGP)},
  vol.~28, pp. 1497--1508, 2009.

\end{thebibliography}

%{\footnotesize \itemsep=-.2cm \printindex } % Index
\end{document}